\documentclass{article}

\usepackage[nonatbib, final]{neurips_2021}

\usepackage[utf8]{inputenc} % allow utf-8 input
\usepackage[T1]{fontenc}    % use 8-bit T1 fonts
\usepackage{hyperref}       % hyperlinks
\usepackage{url}            % simple URL typesetting
\usepackage{booktabs}       % professional-quality tables
\usepackage{amsfonts}       % blackboard math symbols
\usepackage{nicefrac}       % compact symbols for 1/2, etc.
\usepackage{microtype}      % microtypography
\usepackage{xcolor}         % colors
\title{Neural Bellman-Ford Networks: A General Graph Neural Network Framework for Link Prediction}
\author{
  Zhaocheng Zhu\textsuperscript{1,2}, Zuobai Zhang\textsuperscript{1,2}, Louis-Pascal Xhonneux\textsuperscript{1,2}, Jian Tang\textsuperscript{1,3,4} \\
  Mila - Qu\'ebec AI Institute\textsuperscript{1}, Universit\'e de Montr\'eal\textsuperscript{2} \\
  HEC Montr\'eal\textsuperscript{3}, CIFAR AI Chair\textsuperscript{4} \\
  \texttt{\{zhaocheng.zhu, zuobai.zhang, louis-pascal.xhonneux\}@mila.quebec} \\
  \texttt{jian.tang@hec.ca}
}

\usepackage{enumitem}
\usepackage{xspace}
\usepackage{xcolor}
\usepackage{wrapfig}
\usepackage{caption}
\usepackage{amssymb}
\usepackage{multirow}
\usepackage{adjustbox}
\usepackage{subcaption}
\usepackage{algorithm}
\usepackage{algpseudocode}

\usepackage{thmtools}
\usepackage{thm-restate}

\newcommand{\restatableeq}[3]{\label{#3}#2\gdef#1{#2\tag{\ref{#3}}}}

\newtheorem{theorem}{Theorem}
\newcommand{\method}{NBFNet\xspace}

\newcommand{\mathbbm}[1]{\text{\usefont{U}{bbm}{m}{n}#1}}
\newcommand{\best}[1]{\textbf{#1}}
\newcommand{\ozero}{\text{\textcircled{\scriptsize{0}}}}
\newcommand{\oone}{\text{\textcircled{\scriptsize{1}}}}

\newcommand{\edge}[1]{$\langle$\textit{#1}$\rangle$}
\definecolor{myblue}{RGB}{68, 114, 196}
\definecolor{myorange}{RGB}{237, 125, 49}

\usepackage{amsmath}
\usepackage{bm}

\def\eqref#1{equation~\ref{#1}}
\def\1{\bm{1}}

\def\vb{{\bm{b}}}

\def\ve{{\bm{e}}}

\def\vh{{\bm{h}}}

\def\vm{{\bm{m}}}

\def\vq{{\bm{q}}}

\def\vv{{\bm{v}}}
\def\vw{{\bm{w}}}
\def\vx{{\bm{x}}}

\def\mA{{\bm{A}}}

\def\mD{{\bm{D}}}

\def\mW{{\bm{W}}}

\DeclareMathAlphabet{\mathsfit}{\encodingdefault}{\sfdefault}{m}{sl}
\SetMathAlphabet{\mathsfit}{bold}{\encodingdefault}{\sfdefault}{bx}{n}

\def\gE{{\mathcal{E}}}

\def\gG{{\mathcal{G}}}

\def\gL{{\mathcal{L}}}
\def\gM{{\mathcal{M}}}
\def\gN{{\mathcal{N}}}

\def\gP{{\mathcal{P}}}

\def\gR{{\mathcal{R}}}

\def\gV{{\mathcal{V}}}

\DeclareMathOperator*{\topk}{top-k}

\newtheorem{lemma}[theorem]{Lemma}

\usepackage{cancel}

\begin{document}

\maketitle

\begin{abstract}

Link prediction is a very fundamental task on graphs. Inspired by traditional path-based methods, in this paper we propose a general and flexible representation learning framework based on paths for link prediction. Specifically, we define the representation of a pair of nodes as the \emph{generalized sum} of all path representations between the nodes, with each path representation as the \emph{generalized product} of the edge representations in the path.
Motivated by the Bellman-Ford algorithm for solving the shortest path problem, we show that the proposed path formulation can be efficiently solved by the generalized Bellman-Ford algorithm. To further improve the capacity of the path formulation, we propose the Neural Bellman-Ford Network (\method), a general graph neural network framework that solves the path formulation with learned operators in the generalized Bellman-Ford algorithm. The \method parameterizes the generalized Bellman-Ford algorithm with 3 neural components, namely \textsc{Indicator}, \textsc{Message} and \textsc{Aggregate} functions, which corresponds to the boundary condition, \emph{multiplication} operator, and \emph{summation} operator respectively\footnote{Unless stated otherwise, we use \emph{summation} and \emph{multiplication} to refer the generalized operators in the path formulation, rather than the basic operations of arithmetic.}. The \method covers many traditional path-based methods, and can be applied to both homogeneous graphs and multi-relational graphs (e.g., knowledge graphs) in both transductive and inductive settings. Experiments on both homogeneous graphs and knowledge graphs show that the proposed \method outperforms existing methods by a large margin in both transductive and inductive settings, achieving new state-of-the-art results\footnote{Code is available at \url{https://github.com/DeepGraphLearning/NBFNet}\label{fn:repo}}.

\end{abstract}

\section{Introduction}

Predicting the interactions between nodes (a.k.a. link prediction) is a fundamental task in the field of graph machine learning. Given the ubiquitous existence of graphs, such a task has many applications, such as recommender system~\cite{koren2009matrix}, knowledge graph completion~\cite{nickel2015review} and drug repurposing~\cite{ioannidis2020few}.

% traditional methods
Traditional methods of link prediction usually define different heuristic metrics over the paths between a pair of nodes. For example, Katz index~\cite{katz1953new} is defined as a weighted count of paths between two nodes. Personalized PageRank~\cite{page1999pagerank} measures the similarity of two nodes as the random walk probability from one to the other. Graph distance~\cite{liben2007link} uses the length of the shortest path between two nodes to predict their association.
These methods can be directly applied to new graphs, i.e., inductive setting, enjoy good interpretability and scale up to large graphs. However, they are designed based on handcrafted metrics and may not be optimal for link prediction on real-world graphs.

To address these limitations, some link prediction methods adopt graph neural networks (GNNs)~\cite{kipf2016variational, schlichtkrull2018modeling, vashishth2020composition} to automatically extract important features from local neighborhoods for link prediction. Thanks to the high expressiveness of GNNs, these methods have shown state-of-the-art performance. However, these methods can only be applied to predict new links on the training graph, i.e. transductive setting, and lack interpretability. While some recent methods~\cite{zhang2018link, teru2020inductive} extract features from local subgraphs with GNNs and support inductive setting, the scalability of these methods is compromised.

Therefore, we wonder if there exists an approach that enjoys the advantages of both traditional path-based methods and recent approaches based on graph neural networks, i.e., \textbf{generalization in the inductive setting}, \textbf{interpretability}, \textbf{high model capacity} and \textbf{scalability}.

In this paper, we propose such a solution. Inspired by traditional path-based methods, our goal is to develop a general and flexible representation learning framework for link prediction based on the paths between two nodes. Specifically, we define the representation of a pair of nodes as the \emph{generalized sum} of all the path representations between them, where each path representation is defined as the \emph{generalized product} of the edge representations in the path. Many link prediction methods, such as Katz index~\cite{katz1953new}, personalized PageRank~\cite{page1999pagerank}, graph distance~\cite{liben2007link}, as well as graph theory algorithms like widest path~\cite{baras2010path} and most reliable path~\cite{baras2010path}, are special instances of this path formulation with different \emph{summation} and \emph{multiplication} operators. Motivated by the polynomial-time algorithm for the shortest path problem~\cite{bellman1958routing}, we show that such a formulation can be efficiently solved via the generalized Bellman-Ford algorithm~\cite{baras2010path} under mild conditions and scale up to large graphs.

The operators in the generalized Bellman-Ford algorithm---\emph{summation} and \emph{multiplication}---are handcrafted, which have limited flexibility. Therefore, we further propose the Neural Bellman-Ford Networks (\method), a graph neural network framework that solves the above path formulation with learned operators in the generalized Bellman-Ford algorithm. Specifically, \method parameterizes the generalized Bellman-Ford algorithm with three neural components, namely \textsc{Indicator}, \textsc{Message} and \textsc{Aggregate} functions. The \textsc{Indicator} function initializes a representation on each node, which is taken as the boundary condition of the generalized Bellman-Ford algorithm. The \textsc{Message} and the \textsc{Aggregate} functions learn the \emph{multiplication} and \emph{summation} operators respectively.

We show that the \textsc{Message} function can be defined according to the relational operators in knowledge graph embeddings~\cite{bordes2013translating, yang2015embedding, trouillon2016complex, kazemi2018simple, sun2019rotate}, e.g., as a translation in Euclidean space induced by the relational operators of TransE~\cite{bordes2013translating}. The \textsc{Aggregate} function can be defined as learnable set aggregation functions~\cite{zaheer2017deep, xu2018powerful, corso2020principal}. With such parameterization, \method can generalize to the inductive setting, meanwhile achieve one of the lowest time complexity among inductive GNN methods. A comparison of \method and other GNN frameworks for link prediction is showed in Table~\ref{tab:comparison}. With other instantiations of \textsc{Message} and \textsc{Aggregate} functions, our framework can also recover some existing works on learning logic rules~\cite{yang2017differentiable, sadeghian2019drum} for link prediction on knowledge graphs (Table~\ref{tab:semiring_instance}).

% achievement
Our \method framework can be applied to several link prediction variants, covering not only single-relational graphs (e.g., homogeneous graphs) but also multi-relational graphs (e.g., knowledge graphs). We empirically evaluate the proposed \method for link prediction on homogeneous graphs and knowledge graphs in both transductive and inductive settings. Experimental results show that the proposed \method outperforms existing state-of-the-art methods by a large margin in all settings, with an average relative performance gain of 18\% on knowledge graph completion (HITS@1) and 22\% on inductive relation prediction (HITS@10). We also show that the proposed \method is indeed interpretable by visualizing the top-k relevant paths for link prediction on knowledge graphs. 
\vspace{-1em}
\begin{table}[!h]
    \centering
    \caption{Comparison of GNN frameworks for link prediction. The time complexity refers to the \emph{amortized time} for predicting a single edge or triplet. $|\gV|$ is the number of nodes, $|\gE|$ is the number of edges, and $d$ is the dimension of representations. The wall time is measured on FB15k-237 test set with 40 CPU cores and 4 GPUs. We estimate the wall time of GraIL based on a downsampled test set.}
    \label{tab:comparison}
    \begin{adjustbox}{max width=\textwidth}
    \begin{tabular}{lccccc}
         \toprule
         \bf{Method} & \bf{Inductive\footnotemark} & \bf{Interpretable} & \bf{Learned Representation}  & \bf{Time Complexity} & \bf{Wall Time} \\
         \midrule
         VGAE~\cite{kipf2016variational} / &  &  & \multirow{2}{*}{\checkmark} & \multirow{2}{*}{$O(d)$} & \multirow{2}{*}{18 secs} \\
         RGCN~\cite{schlichtkrull2018modeling} \\
         NeuralLP~\cite{yang2017differentiable} / & \multirow{2}{*}{\checkmark} & \multirow{2}{*}{\checkmark} &  & \multirow{2}{*}{$O\left(\frac{|\gE|d}{|\gV|} + d^2\right)$} & \multirow{2}{*}{2.1 mins} \\
         DRUM~\cite{sadeghian2019drum} \\
         SEAL~\cite{zhang2018link} / & \multirow{2}{*}{\checkmark} & & \multirow{2}{*}{\checkmark} & \multirow{2}{*}{$O(|\gE|d^2)$} & \multirow{2}{*}{$\approx$1 month} \\
         GraIL~\cite{teru2020inductive} \\
         \midrule
         \method & \checkmark & \checkmark & \checkmark & $O\left(\frac{|\gE|d}{|\gV|} + d^2\right)$ & 4.0 mins \\
         \bottomrule
    \end{tabular}
    \end{adjustbox}
\end{table}

\footnotetext{We consider the inductive setting where a model can generalize to entirely new graphs without node features.}
\section{Related Work}

Existing work on link prediction can be generally classified into 3 main paradigms: path-based methods, embedding methods, and graph neural networks.

\textbf{Path-based Methods.}
Early methods on homogeneous graphs compute the similarity between two nodes based on the weighted count of paths (Katz index~\cite{katz1953new}), random walk probability (personalized PageRank~\cite{page1999pagerank}) or the length of the shortest path (graph distance~\cite{liben2007link}). SimRank~\cite{jeh2002simrank} uses advanced metrics such as the expected meeting distance on homogeneous graphs, which is extended by PathSim~\cite{sun2011pathsim} to heterogeneous graphs. On knowledge graphs, Path Ranking~\cite{lao2010relational, gardner2015efficient} directly uses relational paths as symbolic features for prediction. Rule mining methods, such as NeuralLP~\cite{yang2017differentiable} and DRUM~\cite{sadeghian2019drum}, learn probabilistic logical rules to weight different paths. Path representation methods, such as Path-RNN~\cite{neelakantan2015compositional} and its successors~\cite{das2016chains, wang2021relational}, encode each path with recurrent neural networks (RNNs), and aggregate paths for prediction. However, these methods need to traverse an exponential number of paths and are limited to very short paths, e.g., $\leq 3$ edges. To scale up path-based methods, All-Paths~\cite{toutanova2016compositional} proposes to efficiently aggregate all paths with dynamic programming. However, All-Paths is restricted to bilinear models and has limited model capacity. Another stream of works~\cite{xiong2017deeppath, das2017go, hildebrandt2020reasoning} learns an agent to collect useful paths for link prediction. While these methods can produce interpretable paths, they suffer from extremely sparse rewards and require careful engineering of the reward function~\cite{lin2018multi} or the search strategy~\cite{shen2018m}. Some other works~\cite{chen2018variational, qu2020rnnlogic} adopt variational inference to learn a path finder and a path reasoner for link prediction.

\textbf{Embedding Methods.}
Embedding methods learn a distributed representation for each node and edge by preserving the edge structure of the graph. Representative methods include DeepWalk~\cite{perozzi2014deepwalk} and LINE~\cite{tang2015line} on homogeneous graphs, and TransE~\cite{bordes2013translating}, DistMult~\cite{yang2015embedding} and RotatE~\cite{sun2019rotate} on knowledge graphs. Later works improve embedding methods with new score functions~\cite{trouillon2016complex, dettmers2018convolutional, kazemi2018simple, sun2019rotate, tang2019orthogonal, zhang2020learning} that capture common semantic patterns of the relations, or search the score function in a general design space~\cite{zhang2020autosf}. Embedding methods achieve promising results on link prediction, and can be scaled to very large graphs using multiple GPUs~\cite{zhu2019graphvite}. However, embedding methods do not explicitly encode local subgraphs between node pairs and cannot be applied to the inductive setting.

\textbf{Graph Neural Networks.}
Graph neural networks (GNNs)~\cite{scarselli2008graph, kipf2016semi, velivckovic2018graph, xu2018powerful} are a family of representation learning models that encode topological structures of graphs. For link prediction, the prevalent frameworks~\cite{kipf2016variational, schlichtkrull2018modeling, davidson2018hyperspherical, vashishth2020composition} adopt an auto-encoder formulation, which uses GNNs to encode node representations, and decodes edges as a function over node pairs. Such frameworks are potentially inductive if the dataset provides node features, but are transductive only when node features are unavailable. Another stream of frameworks, such as SEAL~\cite{zhang2018link} and GraIL~\cite{teru2020inductive}, explicitly encodes the subgraph around each node pair for link prediction. While these frameworks are proved to be more powerful than the auto-encoder formulation~\cite{zhang2020revisiting} and can solve the inductive setting, they require to materialize a subgraph for each link, which is not scalable to large graphs. By contrast, our \method explicitly captures the paths between two nodes for link prediction, meanwhile achieves a relatively low time complexity (Table~\ref{tab:comparison}). ID-GNN~\cite{you2021identity} formalizes link prediction as a conditional node classification task, and augments GNNs with the identity of the source node. While the architecture of \method shares some spirits with ID-GNN, our model is motivated by the generalized Bellman-Ford algorithm and has theoretical connections with traditional path-based methods. There are also some works trying to scale up GNNs for link prediction by dynamically pruning the set of nodes in message passing~\cite{xu2019dynamically, han2020xerte}. These methods are complementary to \method, and may be incorporated into our method to further improve scalability.
\section{Methodology}

In this section, we first define a path formulation for link prediction. Our path formulation generalizes several traditional methods, and can be efficiently solved by the generalized Bellman-Ford algorithm. Then we propose Neural Bellman-Ford Networks to learn the path formulation with neural functions.

\subsection{Path Formulation for Link Prediction}
\label{sec:path_formulation}

We consider the link prediction problem on both knowledge graphs and homogeneous graphs. \pagebreak[0] A knowledge graph is denoted by $\gG=(\gV, \gE, \gR)$, where $\gV$ and $\gE$ represent the set of entities (nodes) and relations (edges) respectively, and $\gR$ is the set of relation types. We use $\gN(u)$ to denote the set of nodes connected to $u$, and $\gE(u)$ to denote the set of edges ending with node $u$. A homogeneous graph $\gG=(\gV, \gE)$ can be viewed as a special case of knowledge graphs, with only one relation type for all edges. Throughout this paper, we use \textbf{bold} terms, $\vw_q(e)$ or $\vh_q(u, v)$, to denote vector representations, and \textit{italic} terms, $w_e$ or $w_{uv}$, to denote scalars like the weight of edge $(u, v)$ in homogeneous graphs or triplet $(u, r, v)$ in knowledge graphs. Without loss of generality, we derive our method based on knowledge graphs, while our method can also be applied to homogeneous graphs.

\textbf{Path Formulation.} Link prediction is aimed at predicting the existence of a query relation $q$ between a head entity $u$ and a tail entity $v$. From a representation learning perspective, this requires to learn a pair representation $\vh_q(u, v)$, which captures the local subgraph structure between $u$ and $v$ w.r.t. the query relation $q$. In traditional methods, such a local structure is encoded by counting different types of random walks from $u$ to $v$~\cite{lao2010relational, gardner2015efficient}. Inspired by this construction, we formulate the pair representation as a \emph{generalized sum} of path representations between $u$ and $v$ with a commutative \emph{summation} operator $\oplus$. Each path representation $\vh_q(P)$ is defined as a \emph{generalized product} of the edge representations in the path with the \emph{multiplication} operator $\otimes$.
\begin{align}
    \restatableeq{\link}{&\vh_q(u, v) = \vh_q(P_1) \oplus \vh_q(P_2) \oplus ... \oplus \vh_q(P_{|\gP_{uv}|}) \vert_{P_i \in \gP_{uv}} \triangleq \bigoplus_{P \in \gP_{uv}} \vh_q(P)}{eqn:link} \\
    \restatableeq{\path}{&\vh_q(P = (e_1, e_2, ..., e_{|P|})) = \vw_q(e_1) \otimes \vw_q(e_2) \otimes ... \otimes \vw_q(e_{|P|}) \triangleq \bigotimes_{i=1}^{|P|} \vw_q(e_i)}{eqn:path}
\end{align}
where $\gP_{uv}$ denotes the set of paths from $u$ to $v$ and $\vw_q(e_i)$ is the representation of edge $e_i$. Note the \emph{multiplication} operator $\otimes$ is not required to be commutative (e.g., matrix multiplication), therefore we define $\bigotimes$ to compute the product following the exact order. Intuitively, the path formulation can be interpreted as a depth-first-search (DFS) algorithm, where one searches all possible paths from $u$ to $v$, computes their representations (Equation~\ref{eqn:path}) and aggregates the results (Equation~\ref{eqn:link}). Such a formulation is capable of modeling several traditional link prediction methods, as well as graph theory algorithms. Formally, Theorem~\ref{thm:katz_index}-\ref{thm:reliable_path} state the corresponding path formulations for 3 link prediction methods and 2 graph theory algorithms respectively. See Appendix~\ref{app:path_formulation} for proofs.
\begin{restatable}{thm}{katz}
    Katz index is a path formulation with $\oplus = +$, $\otimes = \times$ and $\vw_q(e) = \beta w_e$.
    \label{thm:katz_index}
\end{restatable}
\vspace{-1em}
\begin{restatable}{thm}{pagerank}
    Personalized PageRank is a path formulation with $\oplus = +$, $\otimes = \times$ and $\vw_q(e) = \alpha w_{uv} / \sum\nolimits_{v'\in\gN(u)}w_{uv'}$.
    \label{thm:pagerank}
\end{restatable}
\vspace{-1em}
\begin{restatable}{thm}{distance}
    Graph distance is a path formulation with $\oplus = \min$, $\otimes = +$ and $\vw_q(e) = w_e$.
    \label{thm:proximity}
\end{restatable}
\vspace{-1em}
\begin{restatable}{thm}{widest}
    Widest path is a path formulation with $\oplus = \max$, $\otimes = \min$ and $\vw_q(e) = w_e$.
    \label{thm:widest_path}
\end{restatable}
\vspace{-1em}
\begin{restatable}{thm}{reliable}
    Most reliable path is a path formulation with $\oplus = \max$, $\otimes = \times$ and $\vw_q(e) = w_e$.
    \label{thm:reliable_path}
\end{restatable}
\textbf{Generalized Bellman-Ford Algorithm.}
While the above formulation is able to model important heuristics for link prediction, it is computationally expensive since the number of paths grows exponentially with the path length.
Previous works~\cite{neelakantan2015compositional, das2016chains, wang2021relational} that directly computes the exponential number of paths can only afford a maximal path length of 3.
A more scalable solution is to use the generalized Bellman-Ford algorithm~\cite{baras2010path}. Specifically, assuming the operators $\langle\oplus, \otimes\rangle$ satisfy a semiring system~\cite{hebisch1998semirings} with \emph{summation identity} $\ozero_q$ and \emph{multiplication identity} $\oone_q$, we have the following algorithm.
\begin{align}
    \restatableeq{\boundary}{\vh_q^{(0)}(u, v) &\leftarrow \mathbbm{1}_q(u = v)}{eqn:boundary} \\
    \restatableeq{\iteration}{\vh_q^{(t)}(u, v) &\leftarrow \left(\bigoplus_{(x, r, v)\in\gE(v)} \vh_q^{(t-1)}(u, x) \otimes \vw_q(x, r, v)\right) \oplus \vh_q^{(0)}(u, v)}{eqn:iteration}
\end{align}
where $\mathbbm{1}_q(u = v)$ is the \textit{indicator} function that outputs $\oone_q$ if $u = v$ and $\ozero_q$ otherwise. $\vw_q(x, r, v)$ is the representation for edge $e = (x, r, v)$ and $r$ is the relation type of the edge. Equation~\ref{eqn:boundary} is known as the boundary condition, while Equation~\ref{eqn:iteration} is known as the Bellman-Ford iteration. The high-level idea of the generalized Bellman-Ford algorithm is to \textbf{compute the pair representation $\vh_q(u, v)$ for a given entity $u$, a given query relation $q$ and all $v \in \gV$ in parallel}, and reduce the total computation by the distributive property of \emph{multiplication} over \emph{summation}. Since $u$ and $q$ are fixed in the generalized Bellman-Ford algorithm, we may abbreviate $\vh_q^{(t)}(u, v)$ as $\vh^{(t)}_v$ when the context is clear. When $\oplus = min$ and $\otimes = +$, it recovers the original Bellman-Ford algorithm for the shortest path problem~\cite{bellman1958routing}. See Appendix~\ref{app:bellman_ford} for preliminaries and the proof of the above algorithm.
\begin{restatable}{thm}{bellman}
    Katz index, personalized PageRank, graph distance, widest path and most reliable path can be solved via the generalized Bellman-Ford algorithm.
    \label{th:bellman-ford}
\end{restatable}
\vspace{-1em}
\begin{table}[!h]
    \centering
    \caption{Comparison of operators in \method and other methods from the view of path formulation.}
    \label{tab:semiring_instance}
    \begin{adjustbox}{max width=\textwidth}
        \begin{tabular}{llcccc}
            \toprule
            \multirow{2}{*}{\bf{Class}} & \multirow{2}{*}{\bf{Method}} & \bf{\textsc{Message}} & \bf{\textsc{Aggregate}} & \bf{\textsc{Indicator}} & \bf{Edge Representation} \\
            & & $\vw_q(e_i) \otimes \vw_q(e_j)$ & $\vh_q(P_i) \oplus \vh_q(P_j)$ & $\ozero_q$, $\oone_q$ & $\vw_q(e)$ \\
            \midrule
            \multirow{3}{*}{\bf{\shortstack[l]{Traditional\\Link\\Prediction}}}
            & Katz Index~\cite{katz1953new} & $\vw_q(e_i) \times \vw_q(e_j)$ & $\vh_q(P_i) + \vh_q(P_j)$ & $0, 1$ & $\beta w_e$ \\
            & Personalized PageRank~\cite{page1999pagerank} & $\vw_q(e_i) \times \vw_q(e_j)$ & $\vh_q(P_i) + \vh_q(P_j)$ & $0, 1$ & $\alpha w_{uv} / \sum_{v'\in\gN(u)}w_{uv'}$ \\
            & Graph Distance~\cite{liben2007link} & $\vw_q(e_i) + \vw_q(e_j)$ & $\min(\vh_q(P_i), \vh_q(P_j))$ & $+\infty, 0$ & $w_e$ \\
            \midrule
            \multirow{2}{*}{\bf{\shortstack[l]{Graph Theory\\Algorithms}}}
            & Widest Path~\cite{baras2010path} & $\min(\vw_q(e_i), \vw_q(e_j))$ & $\max(\vh_q(P_i), \vh_q(P_j))$ & $-\infty, +\infty$ & $w_e$ \\
            & Most Reliable Path~\cite{baras2010path} & $\vw_q(e_i) \times \vw_q(e_j)$ & $\max(\vh_q(P_i), \vh_q(P_j))$ & $0, 1$ & $w_e$ \\
            \midrule
            \multirow{2}{*}{\bf{Logic Rules}}
            & NeuralLP~\cite{yang2017differentiable} / & \multirow{2}{*}{$\vw_q(e_i) \times \vw_q(e_j)$} & \multirow{2}{*}{$\vh_q(P_i) + \vh_q(P_j)$} & \multirow{2}{*}{0, 1} & Weights learned \\
            & DRUM~\cite{sadeghian2019drum} & & & & by LSTM~\cite{hochreiter1997long} \\
            \midrule
            & \multirow{3}{*}{\method} & Relational operators of & \multirow{3}{*}{\shortstack[c]{Learned set\\aggregators~\cite{corso2020principal}}} & \multirow{3}{*}{\shortstack[c]{Learned indicator\\functions}} & \multirow{3}{*}{\shortstack[c]{Learned relation\\embeddings}} \\
            & & knowledge graph \\
            & & embeddings~\cite{bordes2013translating, yang2015embedding, sun2019rotate} \\
            \bottomrule
        \end{tabular}
    \end{adjustbox}
\end{table}

\subsection{Neural Bellman-Ford Networks}
\label{sec:nbfn}

While the generalized Bellman-Ford algorithm can solve many classical methods (Theorem~\ref{th:bellman-ford}), these methods instantiate the path formulation with handcrafted operators (Table~\ref{tab:semiring_instance}), and may not be optimal for link prediction. To improve the capacity of path formulation, we propose a general framework, Neural Bellman-Ford Networks (\method), to learn the operators in the pair representations.

\begin{wrapfigure}{R}{0.57\textwidth}
\begin{minipage}{0.57\textwidth}
    \vspace{-2.6em}
    \begin{algorithm}[H]
        \footnotesize
        \captionsetup{font=footnotesize}\caption{Neural Bellman-Ford Networks}
        \textbf{Input:} source node $u$, query relation $q$, \#layers $T$ \\
        \textbf{Output:} pair representations $\vh_q(u, v)$ for all $v \in \gV$
        \begin{algorithmic}[1]
            \For{$v \in \gV$} \Comment{Boundary condition}
                \State{$\vh_v^{(0)} \gets \Call{Indicator}{u, v, q}$}
            \EndFor
            \For{$t \gets 1$ to $T$} \Comment{Bellman-Ford iteration}
                \For{$v \in \gV$}
                    \State{$\gM^{(t)}_v \gets \left\{\vh_v^{(0)}\right\}$} \Comment{Message augmentation}
                    \For{$(x, r, v) \in \gE(v)$}
                        \State{$\vm_{(x, r, v)}^{(t)} \gets$ \Call{Message$^{(t)}$}{$\vh_x^{(t-1)}, \vw_q(x, r, v)$}}
                        \State{$\gM^{(t)}_v \gets \gM^{(t)}_v \cup \left\{\vm_{(x, r, v)}^{(t)}\right\}$}
                    \EndFor
                    \State{$\vh_v^{(t)} \gets$ \Call{Aggregate$^{(t)}$}{$\gM^{(t)}_v$}}
                \EndFor
            \EndFor
            \State{\Return{$\vh_v^{(T)}$ as $\vh_q(u, v)$ for all $v \in \gV$}}
        \end{algorithmic}
        \label{alg:framework}
    \end{algorithm}
    \vspace{-2.5em}
\end{minipage}
\end{wrapfigure}

\textbf{Neural Parameterization.} We relax the semiring assumption and parameterize the generalized Bellman-Ford algorithm (Equation~\ref{eqn:boundary} and \ref{eqn:iteration}) with 3 neural functions, namely \textsc{Indicator}, \textsc{Message} and \textsc{Aggregate} functions. The \textsc{Indicator} function replaces the \emph{indicator} function $\mathbbm{1}_q(u = v)$. The \textsc{Message} function replaces the binary \emph{multiplication} operator $\otimes$. The \textsc{Aggregate} function is a permutation invariant function over sets that replaces the n-ary \emph{summation} operator $\bigoplus$. Note that one may alternatively define \textsc{Aggregate} as the commutative binary operator $\oplus$ and apply it to a sequence of messages. However, this will make the parameterization more complicated.

Now consider the generalized Bellman-Ford algorithm for a given entity $u$ and relation $q$. In this context, we abbreviate $\vh_q^{(t)}(u, v)$ as $\vh^{(t)}_v$, i.e., a representation on entity $v$ in the $t$-th iteration. It should be stressed that $\vh^{(t)}_v$ is still a pair representation, rather than a node representation. By substituting the neural functions into Equation~\ref{eqn:boundary} and \ref{eqn:iteration}, we get our Neural Bellman-Ford Networks.
\begin{align}
    \vh^{(0)}_v &\leftarrow \textsc{Indicator}(u, v, q) \\
    \vh^{(t)}_v &\leftarrow \textsc{Aggregate}\left(\left\{\textsc{Message}\left(\vh^{(t-1)}_x, \vw_q(x, r, v)\right) \middle\vert (x, r, v) \in \gE(v)\right\} \cup \left\{\vh^{(0)}_v\right\}\right)
\end{align}
\method can be interpreted as a novel GNN framework for learning pair representations. Compared to common GNN frameworks~\cite{kipf2016variational, schlichtkrull2018modeling} that compute the pair representation as two independent node representations $\vh_q(u)$ and $\vh_q(v)$, \method initializes a representation on the source node $u$, and readouts the pair representation on the target node $v$. Intuitively, our framework can be viewed as a source-specific message passing process, where every node learns a representation conditioned on the source node. The pseudo code of \method is outlined in Algorithm~\ref{alg:framework}.

\textbf{Design Space.} Now we discuss some principled designs for \textsc{Message}, \textsc{Aggregate} and \textsc{Indicator} functions by drawing insights from traditional methods. Note the potential design space for \method is way larger than what is presented here, as one can always borrow \textsc{Message} and \textsc{Aggregate} from the arsenal of message-passing GNNs~\cite{hamilton2017inductive, gilmer2017neural, velivckovic2018graph, xu2018powerful}.

For the \textsc{Message} function, traditional methods instantiate it as natural summation, natural multiplication or min over scalars. Therefore, we may use the vectorized version of summation or multiplication. Intuitively, summation of $\vh^{(t-1)}_x$ and $\vw_q(x, r, v)$ can be interpreted as a translation of $\vh^{(t-1)}_x$ by $\vw_q(x, r, v)$ in the pair representation space, while multiplication corresponds to scaling. Such transformations correspond to the relational operators~\cite{hamilton2018embedding, ren2020query2box} in knowledge graph embeddings~\cite{bordes2013translating, yang2015embedding, trouillon2016complex, kazemi2018simple, sun2019rotate}. For example, translation and scaling are the relational operators used in TransE~\cite{bordes2013translating} and DistMult~\cite{yang2015embedding} respectively. We also consider the rotation operator in RotatE~\cite{sun2019rotate}.

The \textsc{Aggregate} function is instantiated as natural summation, max or min in traditional methods, which are reminiscent of set aggregation functions~\cite{zaheer2017deep, xu2018powerful, corso2020principal} used in GNNs. Therefore, we specify the \textsc{Aggregate} function to be sum, mean, or max, followed by a linear transformation and a non-linear activation. We also consider the principal neighborhood aggregation (PNA) proposed in a recent work~\cite{corso2020principal}, which jointly learns the types and scales of the aggregation function.

The \textsc{Indicator} function is aimed at providing a non-trivial representation for the source node $u$ as the boundary condition. Therefore, we learn a query embedding $\vq$ for $\oone_q$ and define \textsc{Indicator} function as $\mathbbm{1}(u = v) * \vq$. Note it is also possible to additionally learn an embedding for $\ozero_q$. However, we find a single query embedding works better in practice.

The edge representations are instantiated as transition probabilities or length in traditional methods. We notice that an edge may have different contribution in answering different query relations. Therefore, we parameterize the edge representations as a linear function over the query relation, i.e., $\vw_q(x, r, v) = \mW_r \vq + \vb_r$. For homogeneous graphs or knowledge graphs with very few relations, we simplify the parameterization to $\vw_q(x, r, v) = \vb_r$ to prevent overfitting. Note that one may also parameterize $\vw_q(x, r, v)$ with learnable entity embeddings $\vx$ and $\vv$, but such a parameterization cannot solve the inductive setting. Similar to NeuralLP~\cite{yang2017differentiable} \& DRUM~\cite{sadeghian2019drum}, we use different edge representations for different iterations, which is able to distinguish noncommutative edges in paths, e.g., father's mother v.s. mother's father.

\textbf{Link Prediction.} We now show how to apply the learned pair representations $\vh_q(u, v)$ to the link prediction problem. We predict the conditional likelihood of the tail entity $v$ as $p(v | u, q) = \sigma(f(\vh_q(u, v)))$, where $\sigma(\cdot)$ is the sigmoid function and $f(\cdot)$ is a feed-forward neural network. The conditional likelihood of the head entity $u$ can be predicted by $p(u|v, q^{-1}) = \sigma(f(\vh_{q^{-1}}(v, u)))$ with the same model. Following previous works~\cite{bordes2013translating, sun2019rotate}, we minimize the negative log-likelihood of positive and negative triplets (Equation~\ref{eqn:kg_loss}). The negative samples are generated according to Partial Completeness Assumption (PCA)~\cite{galarraga2013amie}, which corrupts one of the entities in a positive triplet to create a negative sample. For undirected graphs, we symmetrize the representations and define $p_q(u, v) = \sigma(f(\vh_q(u, v) + \vh_q(v, u)))$. Equation~\ref{eqn:homo_loss} shows the loss for homogeneous graphs.
\vspace{-0.5em}
\begin{align}
    &\gL_{KG} = -\log p(u, q, v) - \sum_{i=1}^n \frac{1}{n}\log (1 - p(u_i', q, v_i'))
    \label{eqn:kg_loss} \\
    &\gL_{homo} = -\log p(u, v) -
    \sum_{i=1}^n \frac{1}{n}\log (1 - p(u_i', v_i')),
    \label{eqn:homo_loss}
\end{align}
where $n$ is the number of negative samples per positive sample and $(u_i',q,v_i')$ and $(u_i',v_i')$ are the $i$-th negative samples for knowledge graphs and homogeneous graphs, respectively.

\textbf{Time Complexity.} One advantage of \method is that it has a relatively low time complexity during inference\footnote{Although the same analysis can be applied to training on a fixed number of samples, we note it is less instructive since one can trade-off samples for performance, and the trade-off varies from method to method.}.
Consider a scenario where a model is required to infer the conditional likelihood of all possible triplets $p(v | u, q)$. We group triplets with the same condition $u, q$ together, where each group contains $|\gV|$ triplets. For each group, we only need to execute Algorithm~\ref{alg:framework} once to get their predictions. Since a small constant number of iterations $T$ is enough for \method to converge (Table~\ref{tab:num_layer}), Algorithm~\ref{alg:framework} has a time complexity of $O(|\gE|d + |\gV|d^2)$, where $d$ is the dimension of representations. Therefore, the amortized time complexity for a single triplet is $O\left(\frac{|\gE|d}{|\gV|} + d^2\right)$. For a detailed derivation of time complexity of other GNN frameworks, please refer to Appendix~\ref{app:complexity}.
\section{Experiment}

\subsection{Experiment Setup}
\label{sec:exp_setup}

We evaluate \method in three settings, knowledge graph completion, homogeneous graph link prediction and inductive relation prediction. The former two are transductive settings, while the last is an inductive setting. For knowledge graphs, we use FB15k-237~\cite{toutanova2015observed} and WN18RR~\cite{dettmers2018convolutional}. We use the standard transductive splits~\cite{toutanova2015observed, dettmers2018convolutional} and inductive splits~\cite{teru2020inductive} of these datasets. For homogeneous graphs, we use Cora, Citeseer and PubMed~\cite{sen2008collective}. Following previous works~\cite{kipf2016variational, davidson2018hyperspherical}, we split the edges into train/valid/test with a ratio of 85:5:10. Statistics of datasets can be found in Appendix~\ref{app:dataset_stat}. Additional experiments of \method on OGB~\cite{hu2020ogb} datasets can be found in Appendix~\ref{app:ogb}.

\textbf{Implementation Details.}
Our implementation generally follows the open source codebases of knowledge graph completion\footnote{\url{https://github.com/DeepGraphLearning/KnowledgeGraphEmbedding}. MIT license.\label{fn:kg_url}} and homogeneous graph link prediction\footnote{\url{https://github.com/tkipf/gae}. MIT license.\label{fn:homo_url}}. For knowledge graphs, we follow \cite{yang2017differentiable, sadeghian2019drum} and augment each triplet \edge{u, q, v} with a flipped triplet \edge{v, q$^{-1}$, u}. For homogeneous graphs, we follow \cite{kipf2016semi, kipf2016variational} and augment each node $u$ with a self loop \edge{u, u}. We instantiate \method with 6 layers, each with 32 hidden units. The feed-forward network $f(\cdot)$ is set to a 2-layer MLP with 64 hidden units. ReLU is used as the activation function for all hidden layers. We drop out edges that directly connect query node pairs during training to encourage the model to capture longer paths and prevent overfitting. Our model is trained on 4 Tesla V100 GPUs for 20 epochs. We select the models based on their performance on the validation set. See Appendix~\ref{app:implementation} for more details.

\textbf{Evaluation.} We follow the filtered ranking protocol~\cite{bordes2013translating} for knowledge graph completion. For a test triplet \edge{u, q, v}, we rank it against all negative triplets \edge{u, q, v'} or \edge{u', q, v} that do not appear in the knowledge graph. We report mean rank (MR), mean reciprocal rank (MRR) and HITS at N (H@N) for knowledge graph completion. For inductive relation prediction, we follow~\cite{teru2020inductive} and draw 50 negative triplets for each positive triplet and use the above filtered ranking. We report HITS@10 for inductive relation prediction. For homogeneous graph link prediction, we follow~\cite{kipf2016variational} and compare the positive edges against the same number of negative edges. We report area under the receiver operating characteristic curve (AUROC) and average precision (AP) for homogeneous graphs.

\textbf{Baselines.} We compare \method against path-based methods, embedding methods, and GNNs. These include 11 baselines for knowledge graph completion, 10 baselines for homogeneous graph link prediction and 4 baselines for inductive relation prediction. Note the inductive setting only includes path-based methods and GNNs, since existing embedding methods cannot handle this setting.

\subsection{Main Results}

Table~\ref{tab:kg_result} summarizes the results on knowledge graph completion. \method significantly outperforms existing methods on all metrics and both datasets. \method achieves an average relative gain of 21\% in HITS@1 compared to the best path-based method, DRUM~\cite{sadeghian2019drum}, on two datasets. Since DRUM is a special instance of \method with natural summation and multiplication operators, this indicates the importance of learning \textsc{Message} and \textsc{Aggregate} functions in \method. \method also outperforms the best embedding method, LowFER~\cite{amin2020lowfer}, with an average relative performance gain of 18\% in HITS@1 on two datasets. Meanwhile, \method requires much less parameters than embedding methods. \method only uses 3M parameters on FB15k-237, while TransE needs 30M parameters. See Appendix~\ref{app:num_param} for details on the number of parameters.

\begin{table}[!h]
    \centering
    \caption{Knowledge graph completion results. Results of NeuraLP and DRUM are taken from \cite{sadeghian2019drum}. Results of RotatE, HAKE and LowFER are taken from their original papers~\cite{sun2019rotate, zhang2020learning, amin2020lowfer}. Results of the other embedding methods are taken from \cite{sun2019rotate}. Since GraIL has scalability issues in this setting, we evaluate it with 50 and 100 negative triplets for FB15k-237 and WN18RR respectively and report MR based on an unbiased estimation.}
    \label{tab:kg_result}
    \begin{adjustbox}{max width=\textwidth}
        \begin{tabular}{llcccccccccc}
            \toprule
            \multirow{2}{*}{\bf{Class}} & \multirow{2}{*}{\bf{Method}}
            & \multicolumn{5}{c}{\bf{FB15k-237}} & \multicolumn{5}{c}{\bf{WN18RR}} \\
            & & \bf{MR} & \bf{MRR} & \bf{H@1} & \bf{H@3} & \bf{H@10} & \bf{MR} & \bf{MRR} & \bf{H@1} & \bf{H@3} & \bf{H@10} \\
            \midrule
            \multirow{3}{*}{\bf{Path-based}}
            & Path Ranking~\cite{lao2010relational} & 3521 & 0.174 & 0.119 & 0.186 & 0.285 & 22438 & 0.324 & 0.276 & 0.360 & 0.406 \\
            & NeuralLP~\cite{yang2017differentiable} & - & 0.240 & - & - & 0.362 & - & 0.435 & 0.371 & 0.434 & 0.566 \\
            & DRUM~\cite{sadeghian2019drum} & - & 0.343 & 0.255 & 0.378 & 0.516 & - & 0.486 & 0.425 & 0.513 & 0.586 \\
            \midrule
            \multirow{6}{*}{\bf{Embeddings}}
            & TransE~\cite{bordes2013translating} & 357 & 0.294 & - & - & 0.465 & 3384 & 0.226 & - & - & 0.501 \\
            & DistMult~\cite{yang2015embedding} & 254 & 0.241 & 0.155 & 0.263 & 0.419 & 5110 & 0.43 & 0.39 & 0.44 & 0.49 \\
            & ComplEx~\cite{trouillon2016complex} & 339 & 0.247 & 0.158 & 0.275 & 0.428 & 5261 & 0.44 & 0.41 & 0.46 & 0.51 \\
            & RotatE~\cite{sun2019rotate} & 177 & 0.338 & 0.241 & 0.375 & 0.553 & 3340 & 0.476 & 0.428 & 0.492 & 0.571 \\
            & HAKE~\cite{zhang2020learning} & - & 0.346 & 0.250 & 0.381 & 0.542 & - & 0.497 & 0.452 & 0.516 & 0.582 \\
            & LowFER~\cite{amin2020lowfer} & - & 0.359 & 0.266 & 0.396 & 0.544 & - & 0.465 & 0.434 & 0.479 & 0.526 \\
            \midrule
            \multirow{3}{*}{\bf{GNNs}}
            & RGCN~\cite{schlichtkrull2018modeling} & 221 & 0.273 & 0.182 & 0.303 & 0.456 & 2719 & 0.402 & 0.345 & 0.437 & 0.494 \\
            & GraIL~\cite{teru2020inductive} & 2053 & - & - & - & - & 2539 & - & - & - & - \\
            & \method & \best{114} & \best{0.415} & \best{0.321} & \best{0.454} & \best{0.599} & \best{636} & \best{0.551} & \best{0.497} & \best{0.573} & \best{0.666} \\
            \bottomrule
        \end{tabular}
    \end{adjustbox}
\end{table}

\begin{table}[!h]
    \vspace{-1em}
    \centering
    \caption{Homogeneous graph link prediction results. Results of VGAE and S-VGAE are taken from their original papers~\cite{kipf2016variational, davidson2018hyperspherical}.}
    \label{tab:homo_result}
    \begin{adjustbox}{max width=\textwidth}
        \begin{tabular}{llcccccc}
            \toprule
            \multirow{2}{*}{\bf{Class}} & \multirow{2}{*}{\bf{Method}}
                    & \multicolumn{2}{c}{\bf{Cora}} & \multicolumn{2}{c}{\bf{Citeseer}} & \multicolumn{2}{c}{\bf{PubMed}} \\
                &   & \bf{AUROC} & \bf{AP} & \bf{AUROC} & \bf{AP} & \bf{AUROC} & \bf{AP} \\
            \midrule
            \multirow{3}{*}{\bf{Path-based}}
            & Katz Index~\cite{katz1953new} & 0.834 & 0.889 & 0.768 & 0.810 & 0.757 & 0.856 \\
            & Personalized PageRank~\cite{page1999pagerank} & 0.845 & 0.899 & 0.762 & 0.814 & 0.763 & 0.860 \\
            & SimRank~\cite{jeh2002simrank} & 0.838 & 0.888 & 0.755 & 0.805 & 0.743 & 0.829 \\
            \midrule
            \multirow{3}{*}{\bf{Embeddings}}
            & DeepWalk~\cite{perozzi2014deepwalk} & 0.831 & 0.850 & 0.805 & 0.836 & 0.844 & 0.841 \\
            & LINE~\cite{tang2015line} & 0.844 & 0.876 & 0.791 & 0.826 & 0.849 & 0.888 \\
            & node2vec~\cite{grover2016node2vec} & 0.872 & 0.879 & 0.838 & 0.868 & 0.891 & 0.914 \\
            \midrule
            \multirow{5}{*}{\bf{GNNs}}
            & VGAE~\cite{kipf2016variational} & 0.914 & 0.926 & 0.908 & 0.920 & 0.944 & 0.947 \\
            & S-VGAE~\cite{davidson2018hyperspherical} & 0.941 & 0.941 & \best{0.947} & \best{0.952} & 0.960 & 0.960 \\
            & SEAL~\cite{zhang2018link} & 0.933 & 0.942 & 0.905 & 0.924 & 0.978 & 0.979 \\
            & TLC-GNN~\cite{yan2021link} & 0.934 & 0.931 & 0.909 & 0.916 & 0.970 & 0.968 \\
            & \method             & \best{0.956} & \best{0.962} & 0.923 & 0.936 & \best{0.983} & \best{0.982} \\
            \bottomrule
        \end{tabular}
    \end{adjustbox}
\end{table}

\begin{table}[!h]
    \vspace{-1em}
    \centering
    \caption{Inductive relation prediction results (HITS@10). V1-v4 corresponds to the 4 standard versions of inductive splits. Results of compared methods are taken from \cite{teru2020inductive}.}
    \label{tab:inductive_result}
    \footnotesize
    \begin{tabular}{llcccccccc}
        \toprule
        \multirow{2}{*}{\bf{Class}} & \multirow{2}{*}{\bf{Method}} & \multicolumn{4}{c}{\bf{FB15k-237}} & \multicolumn{4}{c}{\bf{WN18RR}} \\
        & & \bf{v1} & \bf{v2} & \bf{v3} & \bf{v4} & \bf{v1} & \bf{v2} & \bf{v3} & \bf{v4} \\
        \midrule
        \multirow{3}{*}{\bf{Path-based}}
        & NeuralLP~\cite{gilmer2017neural} & 0.529 & 0.589 & 0.529 & 0.559 & 0.744 & 0.689 & 0.462 & 0.671 \\
        & DRUM~\cite{sadeghian2019drum} & 0.529 & 0.587 & 0.529 & 0.559 & 0.744 & 0.689 & 0.462 & 0.671 \\
        & RuleN~\cite{meilicke2018fine} & 0.498 & 0.778 & 0.877 & 0.856 & 0.809 & 0.782 & 0.534 & 0.716 \\
        \midrule
        \multirow{2}{*}{\bf{GNNs}}
        & GraIL~\cite{teru2020inductive} & 0.642 & 0.818 & 0.828 & 0.893 & 0.825 & 0.787 & 0.584 & 0.734 \\
        & \method & \best{0.834} & \best{0.949} & \best{0.951} & \best{0.960} & \best{0.948} & \best{0.905} & \best{0.893} & \best{0.890} \\
        \bottomrule
    \end{tabular}
    \vspace{-1em}
\end{table}

Table~\ref{tab:homo_result} shows the results on homogeneous graph link prediction. \method gets the best results on Cora and PubMed, meanwhile achieves competitive results on CiteSeer. Note CiteSeer is extremely sparse (Appendix~\ref{app:dataset_stat}), which makes it hard to learn good representations with NBFNet. One thing to note here is that unlike other GNN methods, \method does not use the node features provided by the datasets but is still able to outperform most other methods. We leave how to effectively combine node features and structural representations for link prediction as our future work.

Table~\ref{tab:inductive_result} summarizes the results on inductive relation prediction. On all inductive splits of two datasets, \method achieves the best result. \method outperforms the previous best method, GraIL~\cite{teru2020inductive}, with an average relative performance gain of 22\% in HITS@10. Note that GraIL explicitly encodes the local subgraph surrounding each node pair and has a high time complexity (Appendix~\ref{app:complexity}). Usually, GraIL can at most encode a 2-hop subgraph, while our \method can efficiently explore longer paths. 

\subsection{Ablation Study}

\textbf{\textsc{Message} \& \textsc{Aggregate} Functions.} Table~\ref{tab:msg_agg} shows the results of different \textsc{Message} and \textsc{Aggregate} functions. Generally, \method benefits from advanced embedding methods (DistMult, RotatE > TransE) and aggregation functions (PNA > sum, mean, max). Among simple \textsc{Aggregate} functions (sum, mean, max), combinations of \textsc{Message} and \textsc{Aggregate} functions (TransE \& max, DistMult \& sum) that satisfy the semiring assumption\footnote{Here semiring is discussed under the assumption of linear activation functions. Rigorously, no combination satisfies a semiring if we consider non-linearity in the model.} of the generalized Bellman-Ford algorithm, achieve locally optimal performance. PNA significantly improves over simple counterparts, which highlights the importance of learning more powerful \textsc{Aggregate} functions.

\textbf{Number of GNN Layers.} Table~\ref{tab:num_layer} compares the results of \method with different number of layers. Although it has been reported that GNNs with deep layers often result in significant performance drop~\cite{li2018deeper, zhao2019pairnorm}, we observe \method does not have this issue. The performance increases monotonically with more layers, hitting a saturation after 6 layers. We conjecture the reason is that longer paths have negligible contribution, and paths not longer than 6 are enough for link prediction.

\textbf{Performance by Relation Category.} We break down the performance of \method by the categories of query relations: one-to-one, one-to-many, many-to-one and many-to-many\footnote{The categories are defined same as \cite{wang2014knowledge}. We compute the average number of tails per head and the average number of heads per tail. The category is \textit{one} if the average number is smaller than 1.5 and \textit{many} otherwise.}. Table~\ref{tab:rel_category} shows the prediction results for each category. It is observed that \method not only improves on easy one-to-one cases, but also on hard cases where there are multiple true answers for the query.

\vspace{-1em}
\begin{table}[!h]
    \caption{Ablation studies of \method on FB15k-237. Due to space constraints, we only report MRR here. For full results on all metrics, please refer to Appendix~\ref{app:ablation}.}
    \label{tab:ablation}
    \footnotesize
    \begin{subtable}[t]{0.52\textwidth}
        \centering
        \caption{Different \textsc{Message} and \textsc{Aggregate} functions. \label{tab:msg_agg}}
        \vspace{-0.5em}
        \begin{adjustbox}{max width=\textwidth}
            \begin{tabular}{lcccc}
                \toprule
                \multirow{2}{*}{\bf{\textsc{Message}}} & \multicolumn{4}{c}{\bf{\textsc{Aggregate}}} \\
                                                  & Sum & Mean & Max & PNA~\cite{corso2020principal} \\
                \midrule
                TransE~\cite{bordes2013translating} & 0.297 & 0.310 & 0.377 & 0.383 \\
                DistMult~\cite{yang2017differentiable} & 0.388 & 0.384 & 0.374 & \bf{0.415} \\
                RotatE~\cite{sun2019rotate}   & 0.392 & 0.376 & 0.385 & \bf{0.414} \\
                \bottomrule
            \end{tabular}
        \end{adjustbox}
    \end{subtable}
    \hspace{0.3em}
    \begin{subtable}[t]{0.46\textwidth}
        \centering
        \caption{Different number of layers. \label{tab:num_layer}}
        \vspace{-0.5em}
        \begin{adjustbox}{max width=\textwidth}
            \begin{tabular}{lcccc}
                \toprule
                \multirow{2}{*}{\bf{Method}} & \multicolumn{4}{c}{\bf{\#Layers ($T$)}} \\
                & 2 & 4 & 6 & 8 \\
                \midrule
                \method & 0.345 & 0.409 & \bf{0.415} & \bf{0.416} \\
                \bottomrule
            \end{tabular}
        \end{adjustbox}
    \end{subtable}
    \centering
    \begin{subtable}[h]{\textwidth}
        \centering
        \caption{Performance w.r.t. relation category. The two scores are the rankings over heads and tails respectively. \label{tab:rel_category}}
        \vspace{-0.5em}
        \begin{adjustbox}{max width=0.8\textwidth}
            \begin{tabular}{lcccc}
                \toprule
                \multirow{2}{*}{\bf{Method}} & \multicolumn{4}{c}{\bf{Relation Category}} \\
                & \bf{1-to-1} & \bf{1-to-N} & \bf{N-to-1} & \bf{N-to-N} \\
                \midrule
                TransE~\cite{bordes2013translating} & 0.498/0.488 & 0.455/0.071 & 0.079/0.744 & 0.224/0.330 \\
                RotatE~\cite{sun2011pathsim} & 0.487/0.484 & 0.467/0.070 & 0.081/0.747 & 0.234/0.338 \\
                \method & \best{0.578}/\best{0.600} & \best{0.499}/\best{0.122} & \best{0.165}/\best{0.790} & \best{0.348}/\best{0.456} \\
                \bottomrule
            \end{tabular}
        \end{adjustbox}
    \end{subtable}
\end{table}

\subsection{Path Interpretations of Predictions}

One advantage of \method is that we can interpret its predictions through paths, which may be important for users to understand and debug the model. Intuitively, the interpretations should contain paths that contribute most to the prediction $p(u, q, v)$. Following local interpretation methods~\cite{baehrens2010explain, zeiler2014visualizing}, we approximate the local landscape of \method with a linear model over the set of all paths, i.e., 1st-order Taylor polynomial. We define the importance of a path as its weight in the linear model, which can be computed by the partial derivative of the prediction w.r.t.\ the path. Formally, the top-k path interpretations for $p(u, q, v)$ are defined as
\begin{equation}
    P_1, P_2, ..., P_k = \topk_{P \in \gP_{uv}} \frac{\partial{p(u, q, v)}}{\partial{P}}
\end{equation}
Note this formulation generalizes the definition of logical rules~\cite{yang2017differentiable, sadeghian2019drum} to non-linear models. While directly computing the importance of all paths is intractable, we approximate them with edge importance. Specifically, the importance of each path is approximated by the sum of the importance of edges in that path, where edge importance is obtained via auto differentiation. Then the top-k path interpretations are equivalent to the top-k longest paths on the edge importance graph, which can be solved by a Bellman-Ford-style beam search. Better approximation is left as a future work.

Table~\ref{tab:visualization} visualizes path interpretations from FB15k-237 test set. While users may have different insights towards the visualization, here is our understanding. 1) In the first example, \method learns soft logical entailment, such as $\emph{impersonate}^{-1} \land \emph{nationality} \implies \emph{nationality}$ and $\emph{ethnicity}^{-1} \land \emph{distribution} \implies \emph{nationality}$. 2) In second example, \method performs analogical reasoning by leveraging the fact that \emph{Florence} is similar to \emph{Rome}. 3) In the last example, \method extracts longer paths, since there is no obvious connection between \emph{Pearl Harbor (film)} and \emph{Japanese language}.

\begin{table}[!h]
    \vspace{-0.5em}
    \centering
    \caption{Path interpretations of predictions on FB15k-237 test set. For each query triplet, we visualize the top-2 path interpretations and their weights. Inverse relations are denoted with a superscript $^{-1}$.}
    \label{tab:visualization}
    \begin{adjustbox}{max width=\textwidth}
        \footnotesize
        \begin{tabular}{ll}
            \toprule
            \bf{Query} & \edge{u, q, v}: \edge{O. Hardy, nationality, U.S.} \\
            \midrule
            0.243 & \edge{O. Hardy, impersonate$^{-1}$, R. Little} $\land$ \edge{R. Little, nationality, U.S.} \\
            0.224 & \edge{O. Hardy, ethnicity$^{-1}$, Scottish American} $\land$ \edge{Scottish American, distribution, U.S.} \\
            \midrule
            \bf{Query} & \edge{u, q, v}: \edge{Florence, vacationer, D.C. Henrie} \\
            \midrule
            0.251 & \edge{Florence, contain$^{-1}$, Italy} $\land$ \edge{Italy, capital, Rome} $\land$ \edge{Rome, vacationer, D.C. Henrie} \\
            0.183 & \edge{Florence, place live$^{-1}$, G.F. Handel} $\land$ \edge{G.F. Handel, place live, Rome} $\land$ \edge{Rome, vacationer, D.C. Henrie} \\
            \midrule
            \bf{Query} & \edge{u, q, v}: \edge{Pearl Harbor (film), language, Japanese} \\
            \midrule
            0.211 & \edge{Pearl Harbor (film), film actor, C.-H. Tagawa} $\land$ \edge{C.-H. Tagawa, nationality, Japan} \\
            & $\land$ \edge{Japan, country of origin, Yu-Gi-Oh!} $\land$ \edge{Yu-Gi-Oh!, language, Japanese} \\
            0.208 & \edge{Pearl Harbor (film), film actor, C.-H. Tagawa} $\land$ \edge{C.-H. Tagawa, nationality, Japan} \\
            & $\land$ \edge{Japan, official language, Japanese} \\
            \bottomrule
        \end{tabular}
    \end{adjustbox}
\end{table}

\section{Discussion and Conclusion}
\label{sec:discussion}

\textbf{Limitations.} There are a few limitations for \method. First, the assumption of the generalized Bellman-Ford algorithm requires the operators $\langle\oplus, \otimes\rangle$ to satisfy a semiring. Due to the non-linear activation functions in neural networks, this assumption does not hold for \method, and we do not have a theoretical guarantee on the loss incurred by this relaxation. Second, \method is only verified on simple edge prediction, while there are other link prediction variants, e.g., complex logical queries with conjunctions ($\land$) and disjunctions ($\lor$)~\cite{hamilton2018embedding, ren2020query2box}. In the future, we would like to how \method approximates the path formulation, as well as apply \method to other link prediction settings.

\textbf{Social Impacts.} Link prediction has a wide range of beneficial applications, including recommender systems, knowledge graph completion and drug repurposing. However, there are also some potentially negative impacts. First, \method may encode the bias present in the training data, which leads to stereotyped predictions when the prediction is applied to a user on a social or e-commerce platform. Second, some harmful network activities could be augmented by powerful link prediction models, e.g., spamming, phishing, and social engineering. We expect future studies will mitigate these issues.
% \section{Conclusion}

\textbf{Conclusion.} We present a representation learning framework based on paths for link prediction. Our path formulation generalizes several traditional methods, and can be efficiently solved via the generalized Bellman-Ford algorithm. To improve the capacity of the path formulation, we propose \method, which parameterizes the generalized Bellman-Ford algorithm with learned \textsc{Indicator}, \textsc{Message}, \textsc{Aggregate} functions. Experiments on knowledge graphs and homogeneous graphs show that \method outperforms a wide range of methods in both transductive and inductive settings.
\section*{Acknowledgements}

We would like to thank Komal Teru for discussion on inductive relation prediction, Guyue Huang for discussion on fused message passing implementation, and Yao Lu for assistance on large-scale GPU training. We thank Meng Qu, Chence Shi and Minghao Xu for providing feedback on our manuscript.

This project is supported by the Natural Sciences and Engineering Research Council (NSERC) Discovery Grant, the Canada CIFAR AI Chair Program, collaboration grants between Microsoft Research and Mila, Samsung Electronics Co., Ltd., Amazon Faculty Research Award, Tencent AI Lab Rhino-Bird Gift Fund and a NRC Collaborative R\&D Project (AI4D-CORE-06). This project was also partially funded by IVADO Fundamental Research Project grant PRF-2019-3583139727. The computation resource of this project is supported by Calcul Qu\'ebec\footnote{\url{https://www.calculquebec.ca/}} and Compute Canada\footnote{\url{https://www.computecanada.ca/}}.

\bibliographystyle{plain}
\bibliography{reference}

\clearpage
\appendix

\section{Path Formulations for Traditional Methods}
\label{app:path_formulation}

Here we demonstrate our path formulation is capable of modeling traditional link prediction methods like Katz index~\cite{katz1953new}, personalized PageRank~\cite{page1999pagerank} and graph distance~\cite{liben2007link}, as well as graph theory algorithms like widest path~\cite{baras2010path} and most reliable path~\cite{baras2010path}.

Recall the path formulation is defined as
\begin{align}
    \link \\
    \path
\end{align}
which can be written in the following compact form
\begin{equation}
    \vh_q(u,v) = \bigoplus_{P\in\gP_{uv}} \bigotimes_{i=1}^{|P|} \vw_q(e_i)
    \label{eqn:compact}
\end{equation}

\subsection{Katz Index}
The Katz index for a pair of nodes $u$, $v$ is defined as a weighted count of paths between $u$ and $v$, penalized by an attenuation factor $\beta\in(0,1)$. Formally, it can be written as
\begin{equation}
    \text{Katz}(u, v)
    =\sum_{t=1}^{\infty} \beta^t \ve_u^\top \mA^t \ve_v
\end{equation}
where $\mA$ denotes the adjacency matrix and $\ve_u$, $\ve_v$ denote the one-hot vector for nodes $u$, $v$ respectively. The term $\ve_u^\top \mA^t \ve_v$ counts all paths of length $t$ between $u$, and $v$ and shorter paths are assigned with larger weights.

\katz*
\begin{proof}
We show that $\text{Katz}(u, v)$ can be transformed into a summation over all paths between $u$ and $v$, where each path is represented by a product of damped edge weights in the path.
Mathematically, it can be derived as
\begin{align}
    \text{Katz}(u, v)
    &=\sum_{t=1}^{\infty} \beta^t \sum_{P \in \gP_{uv}:|P|=t}\prod_{e \in P} w_{e}\\
    &=\sum_{P \in \gP_{uv}}\prod_{e \in P} \beta w_{e}
\end{align}
Therefore, the Katz index can be viewed as a path formulation with the \emph{summation} operator $+$, the \emph{multiplication} operator $\times$ and the edge representations $\beta w_e$.
\end{proof}

\subsection{Personalized PageRank}
The personalized PageRank (PPR) for $u$ computes the stationary distribution over nodes generated by an infinite random walker, where the walker moves to a neighbor node with probability $\alpha$ and returns to the source node $u$ with probability $1-\alpha$ at each step. The probability of a node $v$ from a source node $u$ has the following closed-form solution~\cite{jeh2003scaling}
\begin{equation}
    \text{PPR}(u, v)
    =(1-\alpha)\sum_{t=1}^{\infty} \alpha^t \ve_u^\top (\mD^{-1}\mA)^t \ve_v
\end{equation}
where $\mD$ is the degree matrix and $\mD^{-1}\mA$ is the (random walk) normalized adjacency matrix.
Note that $\ve_u^\top (\mD^{-1}\mA)^t \ve_v$ computes the probability of $t$-step random walks from $u$ to $v$.

\pagerank*
\begin{proof}
We omit the coefficient $1-\alpha$, since it is always positive and has no effect on the ranking of different node pairs.
Then we have 
\begin{align}
    \text{PPR}(u, v)
    &\propto\sum_{t=1}^{\infty} \alpha^t \sum_{P \in \gP_{uv}:|P|=t}\prod_{(a,b) \in P} \frac{w_{ab}}{\sum_{b'\in\gN(a)}w_{ab'}} \\
    &=\sum_{P \in \gP_{uv}}\prod_{(a,b) \in P} \frac{\alpha w_{ab}}{\sum_{b'\in\gN(a)}w_{ab'}}
\end{align}
where the \emph{summation} operator is $+$, the \emph{multiplication} operator is $\times$ and edge representations are random walk probabilities scaled by $\alpha$.
\end{proof}

\subsection{Graph Distance}
Graph distance (GD) is defined as the minimum length of all paths between $u$ and $v$.

\distance*
\begin{proof}
Since the length of a path is the sum of edge lengths in the path, we have
\begin{equation}
    \text{GD}(u, v) = \min_{P \in \gP_{uv}} \sum_{e \in P} w_e
    \label{eqn:shortest_path}
\end{equation}
Here the \emph{summation} operator is $\min$, the \emph{multiplication} operator is $+$ and the edge representations are the lengths of edges.
\end{proof}

\subsection{Widest Path}
The widest path (WP), also known as the maximum capacity path, is aimed at finding a path between two given nodes, such that the path maximizes the minimum edge weight in the path.

\widest*
\begin{proof}
Given two nodes $u$ and $v$, we can write the widest path as
\begin{align}
    \text{WP}(u, v) =\max_{P \in \gP_{uv}} \min_{e \in P} w_e
\end{align}
Here the \emph{summation} operator is $\max$, the \emph{multiplication} operator is $\min$ and the edge representations are plain edge weights.
\end{proof}

\subsection{Most Reliable Path}
For a graph with non-negative edge probabilities, the most reliable path (MRP) is the path with maximal probability from a start node to an end node. This is also known as Viterbi algorithm~\cite{viterbi1967error} used in the maximum a posterior (MAP) inference of hidden Markov models (HMM).

\reliable*
\begin{proof}
For a start node $u$ and an end node $v$, the probaility of their most reliable path is 
\begin{align}
    \text{MRP}(u, v)
    =\max_{P \in \gP_{uv}} \prod_{e \in P} w_e
\end{align}
Here the \emph{summation} operator is $\max$, the \emph{multiplication} operator is $\times$ and the edge representations are edge probabilities.
\end{proof}

\section{Generalized Bellman-Ford Algorithm}
\label{app:bellman_ford}

First, we prove that the path formulation can be efficiently solved by the generalized Bellman-Ford algorithm when the operators $\langle\oplus, \otimes\rangle$ satisfy a semiring.
Then, we show that traditional methods satisfy the semiring assumption and therefore can be solved by the generalized Bellman-Ford algorithm.

\subsection{Preliminaries on Semirings}

Semirings are algebraic structures with two operators, \emph{summation} $\oplus$ and \emph{multiplication} $\otimes$, that share similar properties with the natural summation and the natural multiplication defined on integers. Specifically, $\oplus$ should be commutative, associative and have an identity element $\ozero$. $\otimes$ should be associative and have an identity element $\oone$. Mathematically, the \emph{summation} $\oplus$ satisfies
\begin{itemize}
    \setlength{\parskip}{0pt}
    \setlength{\itemsep}{0pt plus 1pt}
    \item \textbf{Commutative Property.} $a \oplus b = b \oplus a$
    \item \textbf{Associative Property.} $(a \oplus b) \oplus c = a \oplus (b \oplus c)$
    \item \textbf{Identity Element.} $a \oplus \ozero = a$
\end{itemize}
The \emph{multiplication} $\otimes$ satisfies
\begin{itemize}
    \setlength{\parskip}{0pt}
    \setlength{\itemsep}{0pt plus 1pt}
    \item \textbf{Associative Property.} $(a \otimes b) \otimes c = a \otimes (b \otimes c)$
    \item \textbf{Absorption Property.} $a \otimes \ozero = \ozero \otimes a = \ozero$
    \item \textbf{Identity Element.} $a \otimes \oone = \oone \otimes a = a$
\end{itemize}
Additionally, $\otimes$ should be distributive over $\oplus$.
\begin{itemize}
    \setlength{\parskip}{0pt}
    \setlength{\itemsep}{0pt plus 1pt}
    \item \textbf{Distributive Property (Left).} $a \otimes (b \oplus c) = (a \otimes b) \oplus (a \otimes c)$
    \item \textbf{Distributive Property (Right).} $(b \oplus c) \otimes a = (b \otimes a) \oplus (c \otimes a)$
\end{itemize}
Note semirings differ from natural arithmetic operators in two aspects. First, the \emph{summation} operator $\oplus$ does not need to be invertible, e.g., min or max. Second, the \emph{multiplication} operator $\otimes$ does not need to be commutative nor invertible, e.g., matrix multiplication.

\subsection{Generalized Bellman-Ford Algorithm for Path Formulation}
Now we prove that the generalized Bellman-Ford algorithm computes the path formulation when the operators $\langle\oplus, \otimes\rangle$ satisfy a semiring. It should be stressed that the generalized Bellman-Ford algorithm for path problems has been proved in~\cite{baras2010path}, and not a contribution of this paper. Here we apply the proof to our proposed path formulation.

The generalized Bellman-Ford algorithm computes the following iterations for all $v \in \gV$
\begin{align}
    \boundary \\
    \iteration
\end{align}
\begin{lemma}\label{lem:induction}
After $t$ Bellman-Ford iterations, the intermediate representation $\vh^{(t)}_q(u,v)$ aggregates all path representations within a length of $t$ edges for all $v$. That is
\begin{equation}
    \vh_q^{(t)}(u,v) = \bigoplus_{P\in\gP_{uv}:|P|\le t} \bigotimes_{i=1}^{|P|} \vw_q(e_i)
\end{equation}
\end{lemma}
\begin{proof}
We prove Lemma~\ref{lem:induction} by induction. For the base case $t=0$, there is a single path of length $0$ from $u$ to itself and no path to other nodes. Due to the product definition of path representations, a path of length $0$ is equal to the \emph{multiplication} identity $\oone_q$. Similarly, a summation of no path is equal to the \emph{summation} identity $\ozero_q$. Therefore, we have $\vh_q^{(0)}(u, v) = \mathbbm{1}_q(u = v) = \bigoplus_{P\in\gP_{uv}:|P|=0} \bigotimes_{i=1}^{|P|} \vw_q(e_i)$.

For the inductive case $t > 0$, we consider the second-to-last node $x$ in each path if the path has a length larger than $0$. To avoid overuse of brackets, we use the convention that $\otimes$ and $\bigotimes$ have a higher priority than $\oplus$ and $\bigoplus$. 
\begin{align}
    \vh^{(t)}_q(u,v) 
    &= \left(\bigoplus_{(x, r, v) \in \gE(v)}\vh_q^{(t-1)}(u,x) \otimes \vw_q(x, r, v)\right) \oplus \vh^{(0)}_q(u,v)\\
    &= \left[\bigoplus_{(x, r, v) \in \gE(v)}\left(\bigoplus_{P\in\gP_{ux}:|P|\le t-1} \bigotimes_{i=1}^{|P|} \vw_q(e_i)\right) \otimes \vw_q(x, r, v)\right] \oplus \vh^{(0)}_q(u,v)\label{eq:substitution}\\
    &=
    \left\{\bigoplus_{(x, r, v) \in \gE(v)}\left[\bigoplus_{P\in\gP_{ux}:|P|\le t-1}\left(\bigotimes_{i=1}^{|P|} \vw_q(e_i)\right) \otimes \vw_q(x, r, v)\right]\right\} \oplus \vh^{(0)}_q(u,v)\label{eq:distributive}\\
    &= \left(\bigoplus_{P\in\gP_{uv}:1\le|P|\le t} \bigotimes_{i=1}^{|P|} \vw_q(e_i)\right) \oplus \left( \bigoplus_{P\in\gP_{uv}:|P|=0} \bigotimes_{i=1}^{|P|} \vw_q(e_i)\right)\label{eq:associative}\\
    &= \bigoplus_{P\in\gP_{uv}:|P|\le t} \bigotimes_{i=1}^{|P|} \vw_q(e_i),
\end{align}
where Equation~\ref{eq:substitution} substitutes the inductive assumption for $\vh_q^{(t-1)}(u, x)$, Equation~\ref{eq:distributive} uses the distributive property of $\otimes$ over $\oplus$.
\end{proof}

By comparing Lemma~\ref{lem:induction} and Equation~\ref{eqn:compact}, we can see the intermediate representation converges to our path formulation $\lim_{t\rightarrow \infty} \vh_q^{(t)}(u, v) = \vh_q(u, v)$. More specifically, at most $|\gV|$ iterations are required if we only consider simple paths, i.e., paths without repeating nodes. In practice, for link prediction we find it only takes a very small number of iterations (e.g., $T = 6$) to converge, since long paths make negligible contribution to the task.

\subsection{Traditional Methods}

\bellman*
\begin{proof}
Given that the generalized Bellman-Ford algorithm solves the path formulation when $\langle\oplus, \otimes\rangle$ satisfy a semiring, we only need to show that the operators of the path formulations for traditional methods satisfy semiring structures.

Katz index (Theorem~\ref{thm:katz_index}) and personalized PageRank (Theorem~\ref{thm:pagerank}) use the natural summation $+$ and the natural multiplication $\times$, which obviously satisfy a semiring.

Graph distance (Theorem~\ref{thm:proximity}) uses $\min$ for \emph{summation} and $+$ for \emph{multiplication}. The corresponding identities are $\ozero = +\infty$ and $\oone = 0$. It is obvious that $+$ satisfies the associative property and has identity element $0$.
\begin{itemize}
    \setlength{\parskip}{0pt}
    \setlength{\itemsep}{0pt plus 1pt}
    \item \textbf{Commutative Property.} $\min(a, b) = \min(b, a)$
    \item \textbf{Associative Property.} $\min(\min(a, b), c) = \min(a, \min(b, c))$
    \item \textbf{Identity Element.} $\min(a, +\infty) = a$
    \item \textbf{Absorption Property.} $a + \infty = \infty + a = +\infty$
    \item \textbf{Distributive Property (Left).} $a + \min(b, c) = \min(a + b, a + c)$
    \item \textbf{Distributive Property (Right).} $\min(b, c) + a = \min(b + a, c + a)$
\end{itemize}

Widest path (Theorem~\ref{thm:widest_path}) uses $\max$ for \emph{summation} and $\min$ for \emph{multiplication}. The corresponding identities are $\ozero = -\infty$ and $\oone = +\infty$. We have
\begin{itemize}
    \setlength{\parskip}{0pt}
    \setlength{\itemsep}{0pt plus 1pt}
    \item \textbf{Commutative Property.} $\max(a, b) = \max(b, a)$
    \item \textbf{Associative Property.} $\max(\max(a, b), c) = \max(a, \max(b, c))$
    \item \textbf{Identity Element.} $\max(a, -\infty) = a$
    \item \textbf{Associative Property.} $\min(\min(a, b), c) = \min(a, \min(b, c))$
    \item \textbf{Absorption Property.} $\min(a, -\infty) = \min(-\infty, a) = -\infty$
    \item \textbf{Identity Element.} $\min(a, +\infty) = \min(+\infty, a) = a$
    \item \textbf{Distributive Property (Left).} $\min(a, \max(b, c)) = \max(\min(a, b), \min(a, c))$
    \item \textbf{Distributive Property (Right).} $\min(\max(b, c), a) = \max(\min(b, a), \min(c, a))$
\end{itemize}
where the distributive property can be proved by enumerating all possible orders of $a$, $b$ and $c$.

Most reliable path (Theorem~\ref{thm:reliable_path}) uses $\max$ for \emph{summation} and $\times$ for \emph{multiplication}.  The corresponding identities are $\ozero = 0$ and $\oone = 1$, since all path representations are probabilities in $[0, 1]$. It is obvious that $\times$ satisfies the associative property, the absorption property and has identity element $0$.
\begin{itemize}
    \setlength{\parskip}{0pt}
    \setlength{\itemsep}{0pt plus 1pt}
    \item \textbf{Commutative Property.} $\max(a, b) = \max(b, a)$
    \item \textbf{Associative Property.} $\max(\max(a, b), c) = \max(a, \max(b, c))$
    \item \textbf{Identity Element.} $\max(a, 0) = a$
    \item \textbf{Distributive Property (Left).} $a \times \max(b, c) = \max(a \times b, a \times c)$
    \item \textbf{Distributive Property (Right).} $\max(b, c) \times a = \max(b \times a, c \times a)$
\end{itemize}
where the identity element and the distributive property hold for non-negative elements.
\end{proof}

\section{Time Complexity of GNN Frameworks}
\label{app:complexity}

Here we prove the time complexity for \method and other GNN frameworks.

\subsection{\method}

\begin{lemma}
\label{lem:bellmanford}
The time complexity of one \method run (Algorithm~\ref{alg:framework}) is $O(T(|\gE|d + |\gV|d^2))$.
\end{lemma}
\begin{proof}
We break the time complexity by \textsc{Indicator}, \textsc{Message} and \textsc{Aggregate} functions.

\textsc{Indicator} is called $|\gV|$ times, and a single call to \textsc{Indicator} takes $O(d)$ time. \textsc{Message} is called $T(|\gE| + |\gV|)$ times, and a single call to \textsc{Message}, i.e., a relation operator, takes $O(d)$ time. \textsc{Aggregate} is called $T|\gV|$ times over a total of $T|\gE|$ messages with $d$ dimensions. Each call to \textsc{Aggregate} additionally takes $O(d^2)$ time due to the linear transformations in the function.

Therefore, the total complexity is summed to $O(T(|\gE|d + |\gV|d^2))$.
\end{proof}

In practice, we find a small constant $T$ works well for link prediction, and Lemma~\ref{lem:bellmanford} can be reduced to $O(|\gE|d + |\gV|d^2)$ time.

Now consider applying \method to infer the likelihood of all possible triplets. Without loss of generality, assume we want to predict the tail entity for each head entity and relation $p(v | u, q)$. We group triplets with the same condition $u, q$ together, where each group contains $|\gV|$ triplets. For triplets in a group, we only need to execute Algorithm~\ref{alg:framework} once to get their predictions. Therefore, the amortized time for a single triplet is $O\left(\frac{|\gE|d}{|\gV|} + d^2\right)$.

\subsection{VGAE / RGCN}
RGCN is a message-passing GNN applied to multi-relational graphs, with the message function being a per-relation linear transformation. VGAE can be viewed as a special case of RGCN applied to single-relational graphs. The time complexity of RGCN is similar to Lemma~\ref{lem:bellmanford}, except that each call to the message function takes $O(d^2)$ time due to the linear transformation. Therefore, the total complexity is $O(T(|\gE|d^2 + |\gV|d^2))$, where $T$ refers to the number of layers in RGCN. Since $|\gV| \leq |\gE|$, the complexity is reduced to $O(T|\gE|d^2)$\footnote{By moving the linear transformations from the message function to the aggregation function, one can also get an implementation of RGCN with $O(T|\gV||\gR|d^2)$ time, which is better for dense graphs but worse for sparse graphs. For knowledge graph datasets used in this paper, the above $O(T|\gE|d^2)$ implementation is better.}. In practice, $T$ is a small constant and we get $O(|\gE|d^2)$ complexity.

While directly executing RGCN once for each triplet is costly, a smart way to apply RGCN for inference is to first compute all node representations, and then perform link prediction with the node representations. The first step runs RGCN once for $|\gV|^2|\gR|$ triplets, while the second step takes $O(d)$ time. Therefore, the amortized time for a single triplet is $O\left(\frac{|\gE|d^2}{|\gV|^2|\gR|} + d\right)$. For large graphs and reasonable choices of $d$, we have $|\gE|d \leq |\gV|^2|\gR|$ and the amortized time can be reduced to $O(d)$.

\subsection{NeuralLP / DRUM}
DRUM can be viewed as a special case of \method with \textsc{Message} being Hadamard product and \textsc{Aggregate} being natural summation. NeuralLP is a special case of DRUM where the dimension $d$ equals to 1. Since there is no linear transformation in their \textsc{Aggregate} functions, the amortized time complexity for the message passing part is $O\left(\frac{T|\gE|d}{|\gV|}\right)$. Both DRUM and NeuralLP additionally use an LSTM to learn the edge weights for each layer, which additionally costs $O(Td^2)$ time for $T$ layers. $T$ is small and can be ignored like in other methods. Therefore, the amortized time of two parts is summed to $O\left(\frac{|\gE|d}{|\gV|} + d^2\right)$.

\subsection{SEAL / GraIL}
GraIL first extracts a local subgraph surrounding the link, and then applies RGCN to the local subgraph. SEAL can be viewed as a special case of GraIL applied to single-relational graphs. Therefore, their amortized time is the same as that of one RGCN run, which is $O(|\gE|d^2)$.

Note that one may still run GraIL on large graphs by restricting the local subgraphs to be very small, e.g., within 1-hop neighborhood of the query entities. However, this will severely harm the performance of link prediction. Moreover, most real-world graphs are small-world networks, and a moderate radius can easily cover a non-trivial number of nodes and edges, which costs a lot of time for GraIL.

\section{Number of Parameters}
\label{app:num_param}

\begin{table}[!h]
    \centering
    \caption{Number of parameters in \method. The number of parameters only grows with the number of relations $|\gR|$, rather than the number of nodes $|\gV|$ or edges $|\gE|$. For FB15k-237 augmented with flipped triplets, $|\gR|$ is 474. Our best configuration uses $T=6$, $d=32$ and hidden dimension $m=64$.}
    \label{tab:num_param}
    \begin{tabular}{lcccc}
        \toprule
                            & \multicolumn{2}{c}{\bf{\#Parameter}} \\
                            & \bf{Analytic Formula} & \bf{FB15k-237} \\
        \midrule
        \textsc{Indicator}  & $|\gR|d$              & 15,168        \\
        \textsc{Message}    & $T|\gR|d(d+1)$        & 3,003,264     \\
        \textsc{Aggregate}  & $Td(13d+3)$           & 80,448        \\
        $f(\cdot)$          & $m(2d+1)+m+1$         & 4,225         \\
        \midrule
        Total               &                       & 3,103,105     \\
        \bottomrule
    \end{tabular}
\end{table}

One advantage of \method is that it requires much less parameters than embedding methods. For example, on FB15k-237, \method requires 3M parameters while TransE requires 30M parameters. Table~\ref{tab:num_param} shows a break down of number of parameters in \method. Generally, the number of parameters in \method scales linearly w.r.t. the number of relations, regardless the number of entities in the graph, which makes \method more parameter-efficient for large graphs.

\section{Statistics of Datasets}
\label{app:dataset_stat}

Dataset statistics of two transductive settings, i.e., knowledge graph completion and homogeneous graph link prediction, are summarized in Table~\ref{tab:kg_statistics} and \ref{tab:homo_statistics}. Dataset statistics of inductive relation prediction is summarized in Table~\ref{tab:inductive_statistics}.

We use the standard transductive splits~\cite{toutanova2015observed, dettmers2018convolutional} and inductive splits~\cite{teru2020inductive} for knowledge graphs. For homogeneous graphs, we follow previous works~\cite{kipf2016variational, davidson2018hyperspherical} and randomly split the edges into train/validation/test sets with a ratio of 85:5:10. All the homogeneous graphs used in this paper are undirected. Note that for inductive relation prediction, the original paper~\cite{teru2020inductive} actually uses a \emph{transductive valid set} that shares the same set of fact triplets as the training set for hyperparameter tuning. The \emph{inductive test set} contains entities, query triplets and fact triplets that never appear in the training set. The same data split is adopted in this paper for a fair comparison.

\begin{table}[!h]
    \begin{minipage}[t]{0.58\textwidth}
        \centering
        \caption{Dataset statistics for knowledge graph completion.}
        \label{tab:kg_statistics}
        \begin{adjustbox}{max width=\textwidth}
            \begin{tabular}{lccccc}
                \toprule
                \multirow{2}{*}{\bf{Dataset}} & \multirow{2}{*}{\bf{\#Entity}} & \multirow{2}{*}{\bf{\#Relation}} & \multicolumn{3}{c}{\bf{\#Triplet}} \\
                & & & \bf{\#Train} & \bf{\#Validation} & \bf{\#Test} \\
                \midrule
                FB15k-237~\cite{toutanova2015observed} & 14,541 & 237 & 272,115 & 17,535 & 20,466 \\
                WN18RR~\cite{dettmers2018convolutional} & 40,943 & 11 & 86,835 & 3,034 & 3,134 \\
                \bottomrule
            \end{tabular}
        \end{adjustbox}
    \end{minipage}
    \hspace{0.3em}
    \begin{minipage}[t]{0.42\textwidth}
        \centering
        \caption{Dataset statistics for homogeneous link prediction.}
        \label{tab:homo_statistics}
        \begin{adjustbox}{max width=\textwidth}
            \begin{tabular}{lcccc}
                \toprule
                \multirow{2}{*}{\bf{Dataset}} & \multirow{2}{*}{\bf{\#Node}} & \multicolumn{3}{c}{\bf{\#Edge}} \\
                & & \bf{\#Train} & \bf{\#Validation} & \bf{\#Test} \\
                \midrule
                Cora~\cite{sen2008collective} & 2,708 & 4,614 & 271 & 544 \\
                CiteSeer~\cite{sen2008collective} & 3,327 & 4,022 & 236 & 474 \\
                PubMed~\cite{sen2008collective} & 19,717 & 37,687 & 2,216 & 4,435 \\
                \bottomrule
            \end{tabular}
        \end{adjustbox}
    \end{minipage}
\end{table}

\begin{table}[!h]
    \centering
    \caption{Dataset statistics for inductive relation prediction. Queries refer to the triplets that are used as training or test labels, while facts are the triplets used as training or test inputs. In the training sets, all queries are also provided as facts.}
    \label{tab:inductive_statistics}
    \begin{adjustbox}{max width=\textwidth}
        \begin{tabular}{llcccccccccc}
            \toprule
            \multirow{2}{*}{\bf{Dataset}} & & \multirow{2}{*}{\bf{\#Relation}} & \multicolumn{3}{c}{\bf{Train}} & \multicolumn{3}{c}{\bf{Validation}} & \multicolumn{3}{c}{\bf{Test}} \\
            & & & \bf{\#Entity} & \bf{\#Query} & \bf{\#Fact} & \bf{\#Entity} & \bf{\#Query} & \bf{\#Fact} & \bf{\#Entity} & \bf{\#Query} & \bf{\#Fact} \\
            \midrule
            \multirow{4}{*}{FB15k-237~\cite{teru2020inductive}}
            & v1 & 180 & 1,594 & 4,245 & 4,245 & 1,594 & 489 & 4,245 & 1,093 & 205 & 1,993\\
            & v2 & 200 & 2,608 & 9,739 & 9,739 & 2,608 & 1,166 & 9,739 & 1,660 & 478 & 4,145 \\
            & v3 & 215 & 3,668 & 17,986 & 17,986 & 3,668 & 2,194 & 17,986 & 2,501 & 865 & 7,406 \\
            & v4 & 219 & 4,707 & 27,203 & 27,203 & 4,707 & 3,352 & 27,203 & 3,051 & 1,424 & 11,714 \\
            \multirow{4}{*}{WN18RR~\cite{teru2020inductive}}
            & v1 & 9 & 2,746 & 5,410 & 5,410 & 2,746 & 630 & 5,410 & 922 & 188 & 1,618 \\
            & v2 & 10 & 6,954 & 15,262 & 15,262 & 6,954 & 1,838 & 15,262 & 2,757 & 441 & 4,011\\
            & v3 & 11 & 12,078 & 25,901 & 25,901 & 12,078 & 3,097 & 25,901 & 5,084 & 605 & 6,327\\
            & v4 & 9 & 3,861 & 7,940 & 7,940 & 3,861 & 934 & 7,940 & 7,084 & 1,429 & 12,334 \\
            \bottomrule
        \end{tabular}
    \end{adjustbox}
\end{table}

\section{Implementation Details}
\label{app:implementation}

\begin{table}[!h]
    \centering
    \caption{Hyperparameter configurations of \method on different datasets. Adv. temperature corresponds to the temperature in self-adversarial negative sampling~\cite{sun2019rotate}. Note for FB15k-237 and WN18RR, we use the same hyperparameters for their transductive and inductive settings. We find our model configuration is robust across all datasets, therefore we only tune the learning hyperparameters for each dataset. All the hyperparameters are chosen by the performance on the validation set.}
    \label{tab:hyperparameter}
    \begin{adjustbox}{max width=\textwidth}
        \begin{tabular}{llccccc}
            \toprule
            \multicolumn{2}{l}{\bf{Hyperparameter}}
                                 & \bf{FB15k-237} & \bf{WN18RR} & \bf{Cora} & \bf{CiteSeer} & \bf{PubMed} \\
            \midrule
            \multirow{2}{*}{\bf{GNN}}
                                 & \#layer($T$) & 6         & 6      & 6    & 6        & 6      \\
                                 & hidden dim.  & 32        & 32     & 32   & 32       & 32     \\
            \midrule
            \multirow{2}{*}{\bf{MLP}}
                                 & \#layer     & 2         & 2      & 2    & 2        & 2      \\
                                 & hidden dim. & 64        & 64     & 64   & 64       & 64    \\
            \midrule
            \multirow{2}{*}{\bf{Batch}}
                                 & \#positive  & 256       & 128    & 256  & 256    & 64     \\
                                 & \#negative/\#positive($n$) & 32     & 32   & 1      & 1     & 1      \\
            \midrule
            \multirow{4}{*}{\bf{Learning}}
                                 & optimizer     & Adam   & Adam   & Adam   & Adam   & Adam    \\
                                 & learning rate & 5e-3   & 5e-3   & 5e-3   & 5e-3   & 5e-3    \\
                                 & \#epoch       & 20     & 20     & 20     & 20     & 20      \\
                                 & adv. temperature & 0.5 & 1      & -      & -      & -       \\
            \bottomrule
        \end{tabular}
    \end{adjustbox}
\end{table}

Our implementation generally follows the open source codebases of knowledge graph completion\footnote{\url{https://github.com/DeepGraphLearning/KnowledgeGraphEmbedding}. MIT license.} and homogeneous graph link prediction\footnote{\url{https://github.com/tkipf/gae}. MIT license.}. Table~\ref{tab:hyperparameter} lists the hyperparameter configurations for different datasets. Table~\ref{tab:wall_time} shows the wall time of training and inference on different datasets.

\textbf{Data Augmentation.} For knowledge graphs, we follow previous works~\cite{yang2017differentiable, sadeghian2019drum} and augment each triplet \edge{u, q, v} with a flipped triplet \edge{v, q$^{-1}$, u}. For homogeneous graphs, we follow previous works~\cite{kipf2016semi, kipf2016variational} and augment each node $u$ with a self loop \edge{u, u}.

\textbf{Architecture Details.} We apply Layer Normalization~\cite{ba2016layer} and short cut connection to accelerate the training of \method. Layer Normalization is applied after each \textsc{Aggregate} function. The feed-forward network $f(\cdot)$ is instantiated as a MLP. ReLU is used as the activation function for all hidden layers. For undirected graphs, we symmetrize the pair representation by taking the sum of $\vh_q(u, v)$ and $\vh_q(v, u)$.

\textbf{Training Details.} We train \method on 4 Tesla V100 GPUs with standard data parallelism. During training, we drop out edges that directly connect query node pairs to encourage the model to capture longer paths and prevent overfitting. We select the best checkpoint for each model based on its performance on the validation set. The selection criteria is MRR for knowledge graphs and AUROC for homogeneous graphs.

\textbf{Fused Message Passing.} To reduce memory footprint and better utilize GPU hardware, we follow the efficient implementation of GNNs~\cite{huang2020ge} and implement customized PyTorch operators that combines \textsc{Message} and \textsc{Aggregate} functions into a single operation, without creating all messages explicitly. This reduces the memory complexity of NBFNet from $O(|\gE|d)$ to $O(|\gV|d)$.

\begin{table}[!h]
    \centering
    \caption{Wall time of \method on different datasets and in different settings (Table~\ref{tab:kg_result}, \ref{tab:homo_result} and \ref{tab:inductive_result}). For inductive setting, the total time over 4 split versions is reported.}
    \label{tab:wall_time}
    \begin{adjustbox}{max width=\textwidth}
    \begin{tabular}{lccccccc}
        \toprule
        \multirow{2}{*}{\bf{Wall Time}} & \multicolumn{5}{c}{\bf{Transductive}} & \multicolumn{2}{c}{\bf{Inductive}} \\
                       & FB15k-237 & WN18RR & Cora & CiteSeer & PubMed & FB15k-237 & WN18RR \\
        \midrule
        \bf{Training}  & 9.7 hrs  & 4.4 hrs  & 5.5 mins & 5.3 mins & 5.6 hrs & 23 mins & 41 mins \\
        \bf{Inference} & 4.0 mins & 2.4 mins & < 1 sec & < 1 sec & 25 secs & 6 secs & 20 secs \\
        \bottomrule
    \end{tabular}
    \end{adjustbox}
\end{table}

\section{Experimental Results on OGB Datasets}
\label{app:ogb}

To demonstrate the effectiveness of \method on large-scale graphs, we additionally evaluate our method on two knowledge graph datasets from OGB~\cite{hu2020ogb}, ogbl-biokg and WikiKG90M. We follow the standard evaluation protocol of OGB link property prediction, and compute the mean reciprocoal rank (MRR) of the true entity against 1,000 negative entities.

\subsection{Results on ogbl-biokg}

Ogbl-biokg is a large biomedical knowledge graph that contains 93,773 entities, 51 relations and 5,088,434 triplets. We compare \method with 6 embedding methods on this dataset. Note by the time of this work, only embedding methods are available for such large-scale datasets. Table~\ref{tab:ogbl-biokg} shows the results on ogbl-biokg. NBFNet achieves the best result compared to all methods reported on the official leaderboard\footnote{\url{https://ogb.stanford.edu/docs/leader_linkprop/\#ogbl-biokg}} with much fewer parameters. Note the previous best model AutoSF is based on architecture search and requires more computation resource than NBFNet for training.

\begin{table}[!h]
    \centering
    \caption{Knowledge graph completion results on ogbl-biokg. Results of compared methods are taken from the OGB leaderboard.}
    \label{tab:ogbl-biokg}
    \footnotesize
    \begin{tabular}{llccc}
        \toprule
        \bf{Class} & \bf{Method} & \bf{Test MRR} & \bf{Validation MRR} & \bf{\#Params}\\
        \midrule
        \multirow{6}{*}{\bf{Embeddings}}
        & TransE~\cite{bordes2013translating} & 0.7452 & 0.7456 & 187,648,000 \\
        & DistMult~\cite{yang2015embedding} & 0.8043 & 0.8055 & 187,648,000 \\
        & ComplEx~\cite{trouillon2016complex} & 0.8095 & 0.8105 & 187,648,000 \\
        & RotatE~\cite{sun2019rotate} & 0.7989 & 0.7997 & 187,597,000 \\
        & AutoSF~\cite{zhang2020autosf} & 0.8309 & \best{0.8317} & 93,824,000 \\
        & PairRE~\cite{chao2021pairre} & 0.8164 & 0.8172 & 187,750,000 \\
        \midrule
        \bf{GNNs}
        & \method & \best{0.8317} & \best{0.8318} & 734,209\\
        \bottomrule
    \end{tabular}
\end{table}

\subsection{Results on WikiKG90M}

WikiKG90M is an extremely large dataset used in OGB large-scale challenge~\cite{hu2021ogb}, hold at KDD Cup 2021. It is a general-purpose knowledge graph containing 87,143,637 entities, 1,315 relations and 504,220,369 triplets.

To apply \method to such a large scale, we use a bidirectional breath-first-search (BFS) algorithm to sample a local subgraph for each query. Given a query, we generate a $k$-hop neighborhood for each of the head entity and the candidate tail entities, based on a BFS search. The union of all generated neighborhoods is then collected as the sampled graph. With this sampling algorithm, any path within a length of $2k$ between the head entity and any tail candidate is guaranteed to present in the sampled graph. See Figure~\ref{fig:bfs_sampling} for illustration. While a standard single BFS algorithm computing the $2k$-hop neighborhood of the head entity has the same guarantee, a bidirectional BFS algorithm significantly reduces the number of nodes and edges in the sampled graph.

\begin{figure}[!t]
    \centering
    \begin{subfigure}{0.32\textwidth}
        \centering
        \includegraphics[width=\textwidth]{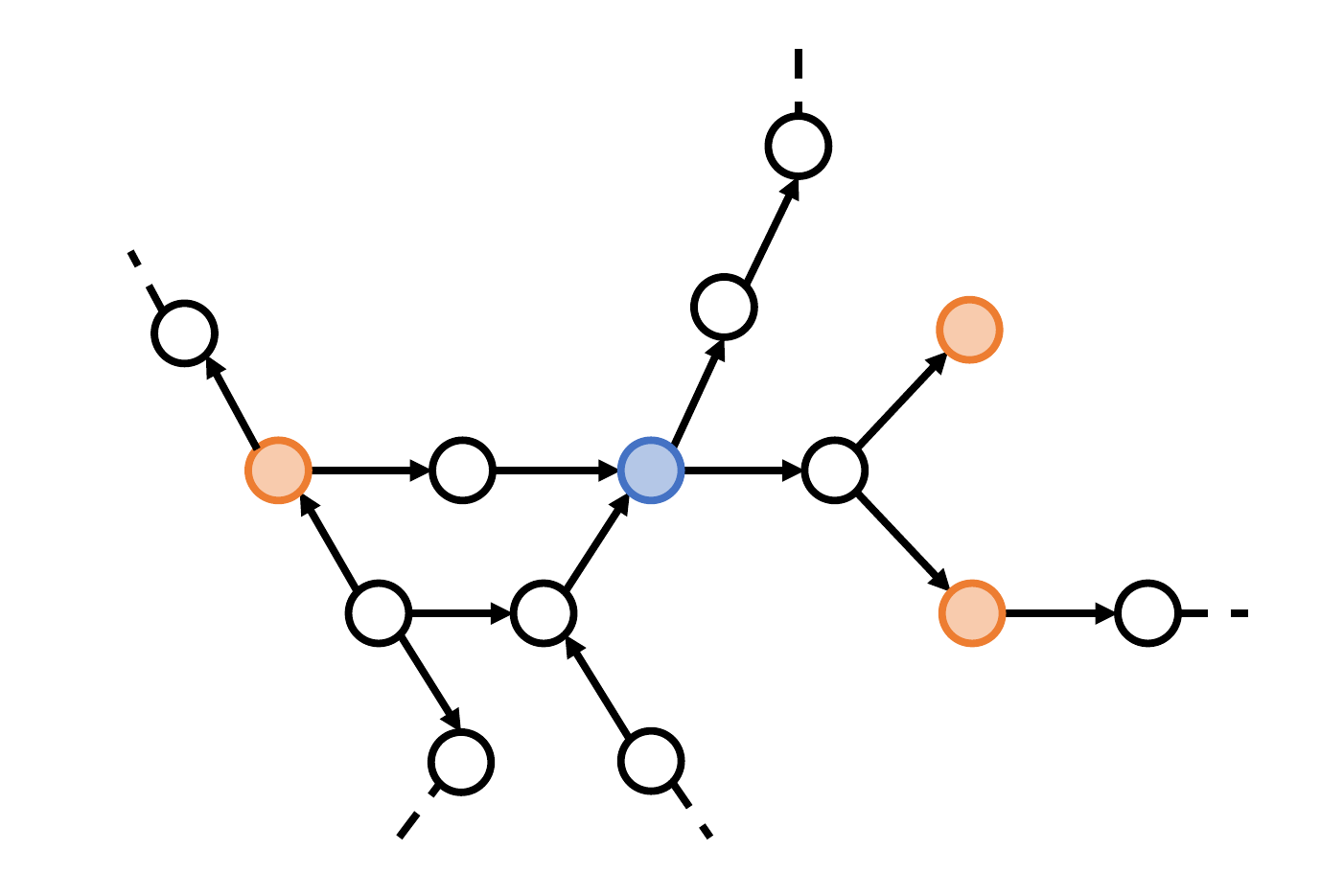}
        \caption{Original graph}
    \end{subfigure}
    \begin{subfigure}{0.32\textwidth}
        \centering
        \includegraphics[width=\textwidth]{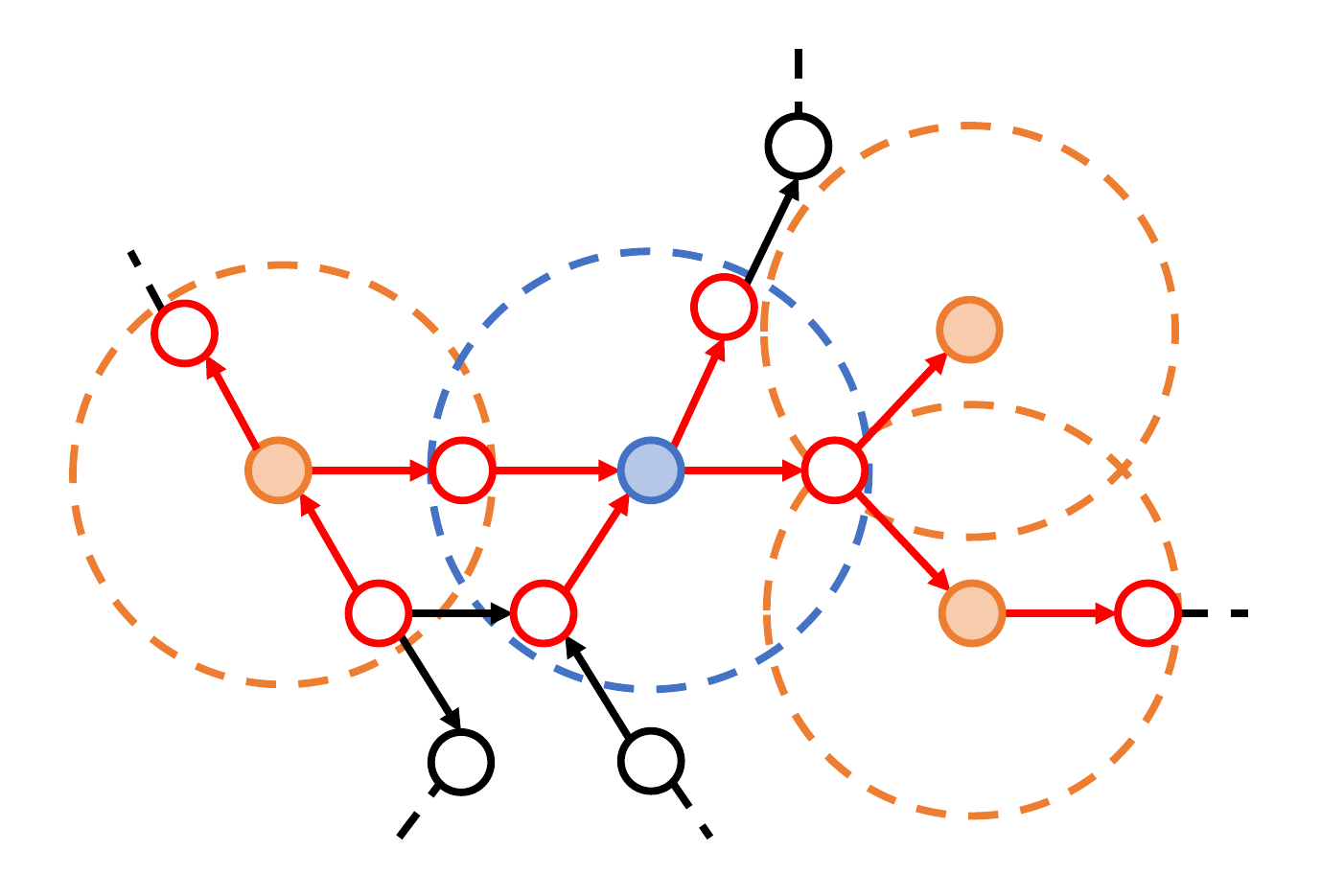}
        \caption{Bidirectional BFS}
    \end{subfigure}
    \begin{subfigure}{0.32\textwidth}
        \centering
        \includegraphics[width=\textwidth]{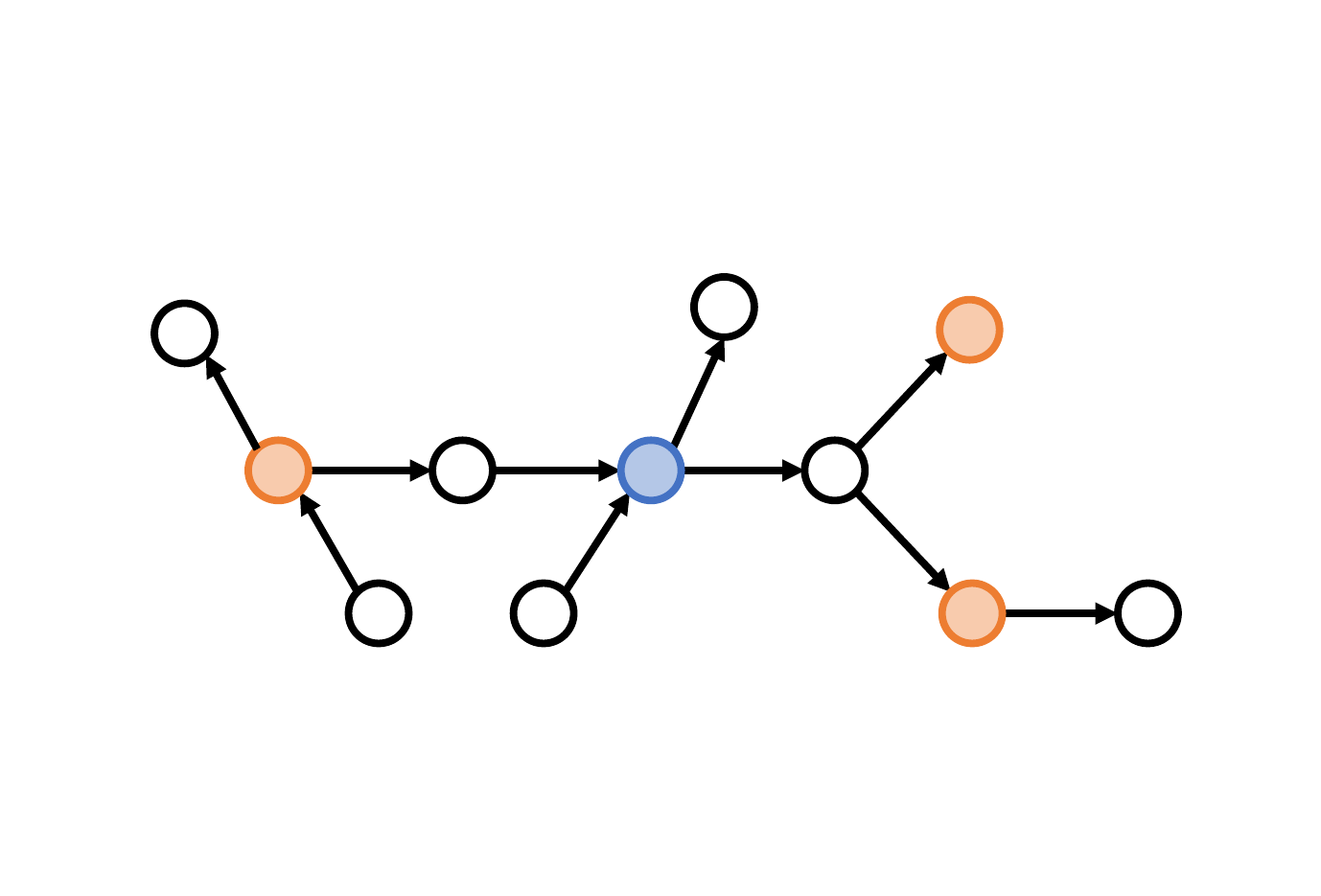}
        \caption{Sampled graph}
    \end{subfigure}
    \caption{Illustration of bidirectional BFS sampling. For a \textcolor{myblue}{head entity} and multiple \textcolor{myorange}{tail candidates}, we use BFS to sample a $k$-hop neighborhood around each entity, regardless of the direction of edges. The neighborhood is denoted by dashed circles. The nodes and edges visited by the BFS algorithm are extracted to generate the sampled graph. Best viewed in color.}
    \label{fig:bfs_sampling}
\end{figure}

We additionally downsample the neighbors when expanding the neighbors of an entity, to tackle entities with large degrees. For each entity visited during the BFS algorithm, we downsample its outgoing neighbors and incoming neighbors to $m$ entities respectively.

Table~\ref{tab:ogb_lsc} shows the results of NBFNet on WikiKG90M validation set. Our best single model uses $k=2$ and $m=100$. While the validation set requires to rank the true entity against 1,000 negative entities, in practice it is not mandatory to draw 1,000 negative samples for each positive sample during training. We find that reducing the negative samples from 1,000 to 20 and increasing the batch size from 4 to 64 provides a better result, although it creates a distribution shift between sampled graphs in training and validation. We leave further research of such distribution shift as a future work.

\begin{table}[!h]
    \begin{minipage}{0.58\textwidth}
    \centering
    \caption{Results of different \textsc{Message} and \textsc{Aggregate} functions on FB15k-237.}
    \begin{adjustbox}{max width=\textwidth}
    \begin{tabular}{llccccc}
        \toprule
        \bf{\textsc{Message}} & \bf{\textsc{Aggregate}} & \bf{MR} & \bf{MRR} & \bf{H@1} & \bf{H@3} & \bf{H@10} \\
        \midrule
        \multirow{4}{*}{TransE~\cite{bordes2013translating}}
        & Sum & 191 & 0.297 & 0.217 & 0.321 & 0.453 \\
        & Mean & 161 & 0.310 & 0.218 & 0.339 & 0.496 \\
        & Max & 135 & 0.377 & 0.282 & 0.415 & 0.565 \\
        & PNA~\cite{corso2020principal} & 129 & 0.383 & 0.288 & 0.420 & 0.568 \\
        \midrule
        \multirow{4}{*}{DistMult~\cite{yang2015embedding}}
        & Sum & 136 & 0.388 & 0.294 & 0.427 & 0.574 \\
        & Mean & 132 & 0.384 & 0.287 & 0.425 & 0.577 \\
        & Max & 136 & 0.374 & 0.279 & 0.412 & 0.563 \\
        & PNA~\cite{corso2020principal} & \bf{114} & \bf{0.415} & \bf{0.321} & \bf{0.454} & \bf{0.599} \\
        \midrule
        \multirow{4}{*}{RotatE~\cite{sun2019rotate}}
        & Sum & 129 & 0.392 & 0.298 & 0.429 & 0.580 \\
        & Mean & 138 & 0.376 & 0.278 & 0.416 & 0.571 \\
        & Max & 139 & 0.385 & 0.290 & 0.423 & 0.572 \\
        & PNA~\cite{corso2020principal} & \bf{117} & \bf{0.414} & \bf{0.323} & \bf{0.454} & \bf{0.593} \\
        \bottomrule
    \end{tabular}
    \end{adjustbox}
    \label{tab:msg_agg_full}
    \end{minipage}
    \hspace{0.3em}
    \begin{minipage}{0.42\textwidth}
    \centering
    \caption{Knowledge graph completion results on WikiKG90M validation set.}
    \label{tab:ogb_lsc}
    \begin{adjustbox}{max width=\textwidth}
        \begin{tabular}{lcc}
            \toprule
            \bf{Model} & Single Model & 6 Model Ensemble \\
            \midrule
            \bf{MRR} & 0.924 & 0.930 \\
            \bottomrule
        \end{tabular}
    \end{adjustbox}
    \vspace{1.5em}
    
    \centering
    \caption{Results of different number of layers on FB15k-237.}
    \begin{adjustbox}{max width=\textwidth}
    \begin{tabular}{lccccc}
        \toprule
        \bf{\#Layers ($T$)} & \bf{MR} & \bf{MRR} & \bf{H@1} & \bf{H@3} & \bf{H@10} \\
        \midrule
        2 & 191 & 0.345 & 0.261 & 0.377 & 0.510 \\
        4 & 119 & 0.409 & 0.315 & 0.450 & 0.592 \\
        6 & \bf{114} & \bf{0.415} & \bf{0.321} & \bf{0.454} & \bf{0.599} \\
        8 & \bf{115} & \bf{0.416} & \bf{0.322} & \bf{0.457} & \bf{0.599} \\
        \bottomrule
    \end{tabular}
    \end{adjustbox}
    \label{tab:num_layer_full}
    \end{minipage}
\end{table}

\section{Ablation Study}
\label{app:ablation}

Table~\ref{tab:msg_agg_full} shows the full results of different \textsc{Message} and \textsc{Aggregate} functions. Table~\ref{tab:num_layer_full} shows the full results of \method with different number of layers.

\end{document}